\documentclass{article}

\usepackage{times}  
\usepackage{helvet}  
\usepackage{courier}  
\usepackage{url}  
\usepackage{graphicx}  
\usepackage{mathtools}
\usepackage{amsmath,amssymb,amsfonts}

\usepackage{caption}
\usepackage{amsthm,dsfont}
\newcommand{\ignore}[1]{}
\usepackage{color}
\usepackage{authblk}
\usepackage{algorithm}
\usepackage{algorithmic}
\usepackage{graphicx}
\usepackage{float} 
\usepackage{multirow}
\usepackage{amsfonts}
\usepackage{comment}
\usepackage{fancyhdr}
\usepackage{rotating}
\usepackage{graphicx}
\usepackage{latexsym}
\usepackage{amsmath}
\usepackage{multirow}
\usepackage{lineno,hyperref}
\usepackage{amsfonts}
\usepackage{comment}
\usepackage{booktabs}
\usepackage{fancyhdr}
\usepackage{algorithm}
\usepackage{algorithmic}
\usepackage{graphicx}
\usepackage{float} 
\usepackage{algorithm}
\usepackage{amsthm}
\usepackage{colortbl}
\newtheorem{theorem}             {Theorem}

\newtheorem{Definition}[theorem]{Definition}
\usepackage{xcolor}
\usepackage{pgfplots}
\usepackage{tikz}
\usepackage{colortbl,soul,xcolor}
\definecolor{bblue}{HTML}{4F81BD}
\definecolor{rred}{HTML}{C0504D}
\definecolor{ggreen}{HTML}{9BBB59}
\definecolor{ppurple}{HTML}{9F4C7C}

\begin{document}
	
	\title{Evolutionary Algorithms for the Chance-Constrained Knapsack Problem}

\author{Yue Xie}
\author{Hirad Assimi}
\author{Aneta Neumann}
\author{Frank Neumann}
\affil{Optimisation and Logistics, School of Computer Science,\\ The University of Adelaide, Adelaide, Australia}
\renewcommand\Authands{ and }
\maketitle

\begin{abstract}
Evolutionary algorithms have been applied to a wide range of stochastic problems. Motivated by real-world problems where constraint violations have disruptive effects, this paper considers the chance-constrained knapsack problem (CCKP) which is a variance of the binary knapsack problem. The problem aims to maximize the profit of selected items under a constraint that the knapsack capacity bound is violated with a small probability. To tackle the chance constraint, we introduce how to construct surrogate functions by applying well-known deviation inequalities such as Chebyshev's inequality and Chernoff bounds. Furthermore, we investigate the performance of several deterministic approaches and introduce a single- and multi-objective evolutionary algorithm to solve the CCKP. In the experiment section, we evaluate and compare the deterministic approaches and evolutionary algorithms on a wide range of instances. Our experimental results show that a multi-objective evolutionary algorithm outperforms its single-objective formulation for all instances and performance better than deterministic approaches according to the computation time. Furthermore, our investigation points out in which circumstances to favour Chebyshev's inequality or the Chernoff bound when dealing with the CCKP. 

\end{abstract}

\section{Introduction}
\label{intro}

Evolutionary algorithms have demonstrated their success in the context of combinatorial optimisation problems. Moreover, evolutionary algorithms have been used for various stochastic problems such as the stochastic job shop problem \cite{HORNG}, the stochastic chemical batch scheduling problem \cite{TILL}, and other dynamic and stochastic problems \cite {Nguyen,Rakshit}. Evolutionary algorithms can obtain good-quality solutions in most cases within a reasonable amount of time in most cases, and can easily apply them to the solution of stochastic problems. However, the mathematical model of the chance-constrained optimisation problem is so complex that it has received little attention in the evolutionary computation literature. In this paper, we develop evolutionary algorithms for the chance-constrained knapsack problem (CCKP), which is a stochastic version of the traditional knapsack problem.

The binary knapsack problem \cite{kellerer} is one of the best-known NP-hard combinatorial optimization problems. The problem is given by a set of $n$ items, each with a weight and a profit, and it objects to maximize the sum of profits of selected items and subjects to the sum of the weight of selected items is less than or equal to the capacity of the knapsack. Generally, in stochastic knapsack problem, the weight and profit of each item are stochastic variables, and the decision of selecting items must be made before any of the random data is realized. Dean et al. \cite{Dean} proved that the stochastic knapsack problem is PSPACE-hard, and there is a significant amount of research on the stochastic knapsack problem in the literature. The general objective is to maximize the expected profit resulting from the assignment of items to the knapsack \cite{Kosuch,Merzifonluoglu}. Due to the difficulty of the stochastic knapsack problem, some researchers investigate approximation results \cite{Bhalgat,Dean,Pike-Burke}. The CCKP studied in this paper assumes the weights of item are stochastic variables conformed to a known continuous probability distribution. The goal of the CCKP is to find a set of items of maximal profit, subject to the condition that the probability with which the total weight will exceeds the capacity bound is less than or equal to a given threshold. Here, the threshold is a small value limiting the probability of the constraint violation. 

Chance-constrained optimization problems~\cite{Charnes,Miller} whose resulting decision ensures the probability of complying with the constraints and the confidence level of being feasible to have received significant attention in the literature. For the general chance-constrained problem, Perkopa et al.~\cite{Prekopa,Andras} proposed a dual-type algorithm, and they investigated the performance of their approaches and compared them with a primal simplex algorithm. Hillier et al.~\cite{Hillier} used linear constraints to generate a procedure for tacking approximate chance constraints. Doerr et al.~\cite{DBLP:journals/corr/abs-1911-11451} investigated submodular optimization problems with chance constraints. They studied the approximation behaviour of greedy algorithms for submodular problems with chance constraints. Chance constraint programming has been widely applied in different disciplines for optimization under uncertainty~\cite{Uryasev}. For example, chance constraint programming has been applied in analogue integrated circuit design~\cite{McConaghy}, mechanical engineering \cite{Mercado}, and other disciplines \cite{liu,poojari}. However, so far, chance constraint programming has received little attention in the evolutionary computation literature~\cite{Zhang}. 

Several research papers that study the stochastic knapsack problem by using chance-constrained programming have been published in the literature. The chance-constrained knapsack problem aims to have a subset of items with maximize profit and remains feasible for the total weight of this set at an acceptable threshold probability~\cite{Goel,Goyal,Kleinberg,klopfenstein}. Kleinberg et al.~\cite{Kleinberg} studied the problem with the weights of items that are only chosen from two possible options. Goel and Indyk \cite{Goel} proposed an algorithm that relaxes the chance constraints by a factor of $(1+\epsilon)$ to solved instances where items have a Poisson distribution or an experimental distribution. Goyal and Ravi \cite{Goyal} investigated the problem where the weights of items conform to the Normal distribution. The proposed linear programming approach can satisfy the chance constraint strictly. Recently, Neumann and Sutton~\cite{fogaruntime} investigated the runtime of the (1+1)~EA for the CCKP and proved that even in the most simple case, it is possible to have local optimal in the search space. Assimi et al.~\cite{DBLP:journals/corr/abs-2002-06766} studied the dynamic chance-constrained knapsack problem and proposed a second objective function to deal with the dynamic capacity of the knapsack. Xie et al.~\cite{xie2020specific} investigated the performance of evolutionary algorithms combined with a heavy-tail mutation operator and a problem-specific crossover operator. They proposed a new multi-objective model for the chance-constrained knapsack problem, which leads to significant performance improvements of multi-objective evolutionary algorithms.

In this paper, for the CCKP, our objective is to find a maximum-value set of items such that the probability of the stochastic weights exceeding the capacity bound is at most $\alpha$. To evaluate a solution concerning the chance constraint, we make use of two inequalities – a Chebyshev’s inequality and a Chernoff bound – to calculate an upper bound on the probability of violating the capacity bound and as surrogate functions for the chance constraint. The probabilistic tools we employ allow us to estimate the probability of a constraint violation mathematically without the need for sampling. To find out conditions when to use either of these two inequalities, we carry out an investigation that shows the expected weights and the variances of a given instance affect the selection between the Chebyshev's inequality and the Chernoff bound. 

We first introduce deterministic approaches named Integer Linear programming (ILP), heuristic approach, and dynamic programming (DP) approach to solve CCKP. We then develop a single-objective approach and a multi-objective approach in terms of solving the CCKP. We consider a simple single-objective evolutionary algorithm named (1+1)~EA and its multi-objective version, GSEMO, both algorithms have been investigated in various studies \cite{Giel,Neumann1,Wegener,Shi} previously. One of the main contributions of our work is to estimate the probability that the total weight of a given solution exceeds the considered weight bound. Such estimation is crucial to determine whether a given solution meets the chance constraint and essential to guide the search of evolutionary computing techniques. 

In experimental investigations, we analyze the results obtained by the proposed approaches dependent on a wide range of knapsack instances associated with confidence level and the uncertainty of the weights. We evaluate the presented deterministic approaches experimentally and compare their performance to the evolutionary algorithms. The considered ILP approach is not able to obtain high-quality solutions in a short time, while the dynamic programming approach requires significantly more computation time than the evolutionary algorithms in all size instances, and the heuristic approach is not being able to obtained solutions in several hours when solving large-scale instances. The proposed multi-objective evolutionary approach can obtain solutions that have a similar quality to the heuristic approach and the DP approach but within a remarkably shorter computation time. The comparison conclusion provides a reasonable justification for using evolutionary optimization when dealing with the chance-constrained knapsack problem.

Moreover, we compare the results obtained by applying different probability inequalities in the deterministic approaches and EAs. The experimental results show that if the threshold of the probability is small, then the performance of the approaches using the Chernoff bound is better than that of Chebyshev's inequality for both single-objective and multi-objective algorithms. Moreover, the experimental results show that the performance of GSEMO is significantly better than (1+1)~EA for all instances. 

This paper is an extension of a conference paper published at GECCO 2019~\cite{Yue}. The conference paper only studies the performance of a single-objective evolutionary algorithm and a multi-objective evolutionary algorithm on the chance-constrained knapsack problem. 
The extension consists of deterministic approaches presented in Section~\ref{sec:deter}.

The remainder of this paper is organized as follows. We introduce the formulation of the chance-constrained knapsack problem and the surrogate functions of the chance constraint in Section \ref{sec:formulation}. In Section \ref{sec:deter} and Section \ref{sec:EAs}, we present the deterministic approaches and the evolutionary algorithms to solve the problem, respectively. Computational experiments and the investigation of the results are described in Section \ref{sec:experimental}, followed by a conclusion in Section \ref{sec:conclusion}.

\section{Problem formulation and surrogate functions of CCKP}
\label{sec:formulation}

In this section, we present the definition of the chance-constrained knapsack problem and the surrogate functions to replace the chance constraint. The surrogate functions are generated by using best-known probability tails Chebyshev's inequality and Chernoff bounds to deal with the chance constraint. We also show the comparisons between these two estimated methods.

\subsection{Problem Definition}

We assume the weights of items are independent of each other, with each weight $w_i$ corresponding to expected value $a_i$ and standard variance $\sigma_i$, and a knapsack capacity $C >0$. We encode a solution as a vector of 0-1 decision variables $X=\{x_1,...,x_n\}$, where $x_i=1$ selects the $i$-th item. Let $W(X)=\sum_{i=1}^n {w_i x_i}$ be the total weight of a given solution $X=\{x_i,...,x_n\}$, with $E[W(X)]=\sum_{i=1}^n {a_i x_i}$ denoting the expected weight of the solution derived by linearity of expectation. Furthermore,  ${Var}(X)=\sum_{i=1}^n {\sigma_i^2 x_i}$ denotes the variance of the weight under the assumption that the variables of items are independent. The goal of the problem is to find a sub-set of items with maximized profit, and the probability of violating the capacity bound is less than or equal to a given threshold denoted as $\alpha$. The CCKP can be formulated as follows:

\begin{align}
	\mathbf{max} \quad & P(X)=\sum_{i \in N} p_ix_i\\
	\mathbf{s.t.}\quad &P_r\left(\sum_{i \in N} w_i x_i \geq C\right)\leq \alpha
	\label{chance} \\
	 & x_i\in \{0,1\}^n
\end{align}

We then present the surrogate functions of the chance constraint \eqref{chance}, which are generated by using Chebyshev's inequality and Chernoff bounds \cite{Raghavan} to deal with the chance constraint.

\subsection{Surrogate Functions}

In this section, we use Chebyshev's inequality and Chernoff bounds \cite{Doerr} to construct the available surrogate that translates to a guarantee on the feasibility of the chance-constrained imposed by Equation \eqref{chance}.

\subsubsection{Chebyshev's inequality}
\label{Chebyshev's inequality}

Firstly, we use the Chebyshev's inequality \eqref{theo:cheb} to reformulate the chance constraint \eqref{chance}. The inequality has great utility for being applied to any probability distribution with known expectation and variance. The Chebyshev's inequality considers tails for upper bound and lower bounds. Since this paper only considers the violation of the capacity bound $C$, we use a one-side Chebyshev's inequality which is known as the Chebyshev-Cantelli inequality. To simplify the presentation, we still use the term Chebyshev's inequality in the following.

\begin{theorem}[Chebyshev's inequality]
\label{thm:cheb}
  Let $X$ be a random variable with expectation $E(X)$ and variance $Var(X)$. Then for any $k\in \mathbb{R}^+$, 
  \begin{equation}
    P_r(X\geq E(X) +k)\leq \frac{Var(X)}{Var(X)+k^2}.
    \label{theo:cheb}
  \end{equation}
\end{theorem}

To match the expression of the Chebyshev's inequality, we set $E[W(X)] +k =C$, then we have a general formula to calculate the upper bound of the chance constraint as follows:

\begin{equation}
\begin{split}
P_r\left(W(X)\geq C\right)\leq \frac{{Var}(X)}{{Var}(X)+\left(C- E[W(X)]\right)^2}.
\end{split}
\label{org:chebyshev}
\end{equation}

Hence, the constraint \eqref{chance} can be reformulated as follows:
\begin{equation}
\begin{split}
\frac{{Var}(X)}{{Var}(X)+\left(C- E[W(X)]\right)^2} \leq \alpha.
\end{split}
\label{chebyshev}
\end{equation}

\begin{Definition}
Given a solution $X$ with stochastic weight $W(X)$, we call the $E[W(X)]+\sqrt{Var(X)\frac{1-\alpha}{\alpha}}$ the surrogate weight of $X$, denoted by $S(X)$.
\end{Definition}

\begin{theorem}
\label{the:Cheb surrogate weight}
If $X$ is a solution vector with surrogate weight $S(X)=E[W(X)]+\sqrt{Var(X)\frac{1-\alpha}{\alpha}}$, then the chance constraint stated in Equation \eqref{chance} is satisfied when $S(X)\leq B$ according to Chebyshev's inequality, for all $\alpha\in (0,1)$.
\end{theorem}

Furthermore, we consider two special cases where each item in the first case has a uniform distribution and takes value in $[a_i -\delta, a_i + \delta]$, which is named the additive uniform distribution. In the second case, each item takes value in $[(1-\beta)a_i, (1+\beta)a_i]$, having a uniform distribution called the multiplicative uniform distribution. Here, $\delta$ and $\beta$ define the uncertainty of the weights of items. For the variable which has a uniform distribution on the interval $[g,h]$, the expectation of this variable is $\mu=\frac{g+h}{2}$ and the variance is $\sigma^2 =\frac{(h-g)^2}{12}$. 
Applying Chebyshev's inequality to the chance constraint, we require

\begin{equation}
  P_r\left( W(X)\geq C\right)\leq\frac{\delta^2\sum_{i \in N} x_i}{\delta^2\sum_{i \in N} x_i+3(C- \sum_{i \in N} a_i x_i)^2} \leq \alpha
  \label{cheby: uniform delta}
\end{equation}
for the additive uniform distribution and

\begin{equation}
  P_r\left( W(X)\geq C\right)\leq\frac{\beta^2 \sum_{i \in N} {a_i}^2 x_i}{\beta^2 \sum_{i \in N} {a_i}^2 x_i+3(C- \sum_{i \in N} a_i x_i)^2} \leq \alpha
  \label{cheby: uniform beta}
\end{equation}
for the multiplicative uniform distribution.
When each weight $w_i$ is chosen according to the Normal distribution $N(a_i, \sigma^2_i)$ and are independent of each other, we get
\begin{equation}
 P_r\left(W(X) \geq C\right)\leq \frac{\sum_{i \in N} \sigma_i^2 x_i}{\sum_{i \in N} \sigma_i^2 x_i+\left(C- \sum_{i \in N} a_i x_i\right)^2} \leq \alpha
\label{cheb:normal:org}
\end{equation}

\subsubsection{Chernoff bounds}
Chernoff bounds provide sharper tails with exponential decay behavior, those bounds are sharper than other known tail bounds such as Markov inequality and Chebyshev's inequality. Chernoff bounds assume that variables are independent and take on values in $[0,1]$. There are several types of Chernoff bounds, and we use the following one, which can be found in \cite{Doerr}.

\begin{theorem}[Chernoff bound]
\label{thm:cher}
  Let $X_1,...,X_n$ be independent random variables taking values in $[0,1]$. Let $X=\sum_{i \in N} X_i$. Let $\epsilon\geq 0$. Then
  \begin{equation}
   P_r(X\geq(1+\epsilon)E(X)) \leq \left(\frac{e^\epsilon}{(1+\epsilon)^{(1+\epsilon)}}\right)^{E(X)}.
   \label{func:cher}
  \end{equation}
\end{theorem}

Function \eqref{func:cher} can be used in the case that all random variables are independent and have their value in $[0,1]$. The Chernoff bound is used to calculate the probability of violating the constraint incorporate with a surrogate function.

\begin{theorem}
\label{chernoff:variance}
  Let the stochastic weights $w_1, \ldots, w_n$ be independent variables with expected values $a_i, \ldots, a_n$, respectively. Let $C >0$ be the capacity of the knapsack, $\sum_{i \in N} w_i x_i$ denotes the total weight of a solution $X=\{x_1,...,x_n\}$, and $E[W(X)]=\sum_{i \in N} a_i x_i$ be the expected weight of the solution, Furthermore, let $\delta \geq 0$ be the uncertainty of the stochastic weights. Then we have the following function.
  \begin{equation}
  \label{chernoff:fun}
   P_r\left(\sum_{i \in N} w_i x_i\geq C\right)
\leq \left(\frac{e^{\frac{C-E[W(X)]}{\delta \sum_{i \in N} x_i}}}{\left(\frac{\delta \sum_{i \in N} x_i +C-E[W(X)]}{\delta\sum_{i \in N} x_i}\right)^{\frac{\delta \sum_{i \in N} x_i +C-E[W(X)]}{\delta\sum_{i \in N} x_i}}}\right)^{\frac{1}{2}\sum_{i \in N} x_i}
  \end{equation}
\end{theorem}

\begin{proof}
In the chance-constrained knapsack problem, we can apply the Chernoff bound in a unique setting that lets $w_i\in[a_i-\delta, a_i+\delta]$ and conform to a uniform distribution. All random variables have the same uncertainty $\delta$ but different initial and final boundaries. We than normalize the stochastic weights to make them chosen values in $[0,1]$. Therefore, we set
  \begin{displaymath}
     y_i=\frac{w_i-(a_i-\delta)}{2\delta}\in [0,1], \, Y(X)=\sum_{i \in N} {y_ix_i}. 
  \end{displaymath}
   Then we have,
  \begin{displaymath}
    E_W(y_i)=\frac{a_i-(a_i -\delta)}{2\delta}=\frac{1}{2}
  \end{displaymath}
  and 
  \begin{displaymath}
   E_W[Y(X)]=\frac{E[W(X)]-(\sum_{i \in N} a_ix_i-\sum_{i \in N} \delta x_i)}{2\delta}=\frac{1}{2}\sum_{i \in N} x_i.
  \end{displaymath}
  We introduce $Y(X)$ in the Chernoff bound and we have
   \begin{equation}
   \begin{split}
    \left(\frac{e^\epsilon}{(1+\epsilon)^{(1+\epsilon)}}\right)^{E_W[Y(X)]} &\geq P_r\left[Y(X)\geq(1+\epsilon)E_W[Y(X)]\right]\\
&= Pr\left[\sum_{i \in N}\frac{w_i-(a_i-\delta)}{2\delta}x_i \geq (1+\epsilon)\frac{\sum_{i \in N} x_i}{2}\right]\\
&= Pr\left[\sum_{i \in N} w_i x_i -\sum_{i \in N} (a_i-\delta)x_i \geq (1+\epsilon)\delta \sum_{i \in N} x_i\right]\\
&= Pr\left[\sum_{i \in N}w_i x_i\geq \epsilon \delta \sum_{i \in N} x_i +\sum_{i \in N} a_i x_i\right].\nonumber
  \end{split}
  \end{equation}
Now, let $ C=\epsilon \delta \sum_{i \in N} x_i +\sum_{i \in N} a_i x_i$, we have $\epsilon=\frac{C-E[W(X)]}{\delta \sum_{i \in N} x_i}$. We substitute $C$ and $\epsilon$ into the last expression, which completes the proof.
\end{proof}

The proof can be distinguished by the following two characteristics. On the one hand, we study the random variable $Y$ rather than $W$, on the other hand, the interval lengths we discussed are the same for all stochastic weights. We reformulate the chance constraint by using the Chernoff bound to estimate the upper bound of the probability of violating the capacity. It should be noted that the interval of all weights should be the equivalent. The surrogate function of the chance constraint is as follows:

\begin{equation}
\label{chernoff:con}
\resizebox{.9\hsize}{!}{
$P_r\left(W(X)\geq C \right) \leq \left(\frac{e^{\frac{C-\sum_{i \in N} a_i x_i}{\delta \sum_{i \in N} x_i}}}{\left(\frac{\delta \sum_{i \in N} x_i +C-\sum_{i \in N} a_i x_i}{\delta\sum_{i \in N} x_i}\right)^{\frac{\delta \sum_{i \in N} x_i +C-\sum_{i \in N} a_i x_i}{\delta\sum_{i \in N} x_i}}}\right)^{\frac{1}{2}\sum_{i \in N} x_i} \leq \alpha$}
\end{equation}

\begin{theorem}
\label{the:Chernoff surrogate weight}
If $X$ is a solution vector with surrogate weight 
$$S'(X)= E[W(X)]+\sqrt{-\frac{2}{\sum_{in \in N}x_i}\ln{\alpha}} \delta \sqrt{\sum_{in \in N}x_i},$$ then the chance constraint stated in Equation \eqref{chance} is satisfied when $S'(X)\leq B$ according to Chernoff bound, and let $w_i\in [a_i-\delta, a_i+\delta]$, for all $\alpha\in (0,1)$.
\end{theorem}

\begin{proof}
Let $\epsilon = \frac{C-E[W(X)]}{\delta \sum_{i \in N}x_i}\geq 0$, by Theorem \ref{chernoff:variance}, this is bounded above by

\begin{align*}
&\left(\frac{e^\epsilon}{(1+\epsilon)^{(1+\epsilon)}}\right)^{\frac{\sum_{i \in N}x_i}{2}}\leq \alpha\\
\Longleftrightarrow \;
& \frac{e^\epsilon}{(1+\epsilon)^{(1+\epsilon)}} \leq (\alpha)^{\frac{2}{\sum_{i \in N}x_i}} \\
 \Longleftrightarrow \;& 
 e^\epsilon \leq (1+\epsilon)^{(1+\epsilon)}(\alpha)^{\frac{2}{\sum_{i \in N}x_i}}\\
\Longleftrightarrow \;&  
\epsilon \leq \frac{2}{\sum_{i \in N}x_i} \ln{\alpha} + (1+\epsilon) \ln{(1+\epsilon)} \\
\Longleftrightarrow \;& 
\epsilon-(1+\epsilon)\ln{(1+\epsilon)} \leq \frac{2}{\sum_{i \in N}x_i} \ln{\alpha} \\
\Longrightarrow \;& 
\epsilon -(1+\epsilon)\epsilon \leq  \frac{2}{\sum_{i \in N}x_i} \ln{\alpha} \\
\Longleftrightarrow \;& 
-\epsilon^2 \leq  \frac{2}{\sum_{i \in N}x_i} \ln{\alpha} \\
\Longleftrightarrow \;& 
\epsilon^2 \geq  -\frac{2}{\sum_{i \in N}x_i} \ln{\alpha} \\
\Longleftrightarrow \;& 
\epsilon \geq \sqrt{-\frac{2}{\sum_{i \in N}x_i} \ln{\alpha}}. \nonumber
\end{align*}

Hence, we have $\frac{C-E[W(X)]}{\delta \sum_{i \in N}x_i }\geq \sqrt{-\frac{2}{\sum_{i \in N}x_i} \ln{\alpha}}$, and  $C \geq E[W(X)]+\sqrt{-\frac{2}{\sum_{i \in N}x_i}\ln{\alpha}}\delta \sqrt{\sum_{i \in N}x_i}$ where we have used the claimed surrogate weight.
\end{proof}

\subsubsection{Comparison of tail inequalities}

Next, we theoretically investigate the effectiveness of Chebyshev's inequality and the Chernoff bound on tacking the chance constraint. The goal of this analysis is to examine which estimation method outperforms the other under various conditions. Let $p_{\mathtt{Cher}}(X)$ denotes the probability obtained by the Chernoff bound and $p_{\mathtt{Cheb}}(X)$ be the probability calculated by the Chebyshev's inequality. The following theorem states a condition under which one inequality should be preferred over the other.

\begin{theorem}
Let $X$ be a solution with the expected weight $E[W(X)]$ and the variance of weight ${Var}(X)$, let $\epsilon = \frac{C-E[W(X)]}{\delta\sum_{i \in N} x_i}$. We have $p_{\mathtt{Cher}}(X) \leq p_{\mathtt{Cheb}}(X)$ if and only if
\begin{equation}
  \frac{\left(\frac{e^\epsilon}{(1+\epsilon)^{1+\epsilon}}\right)^{E[W(X)]}\left(\epsilon E[W(X)]\right)^2}{1-\left(\frac{e^\epsilon}{(1+\epsilon)^{1+\epsilon}}\right)^{E[W(X)]}} \leq {Var}(X).
  \end{equation}
 \label{theo:compare}
\end{theorem}

\begin{proof}
Using the variable $\epsilon$ from the Chernoff bound, we set $k=\epsilon E[W(X)]$ in the Chebyshev's inequality.
Then, we have
\begin{align*}
&\left(\frac{e^\epsilon}{(1+\epsilon)^{1+\epsilon}}\right)^{E[W(X)]} \leq  \frac{{Var}(X)}{{Var}(X) +(\epsilon E[W(X)])^2}\\
\Longleftrightarrow \;& \left(\frac{e^\epsilon}{(1+\epsilon)^{1+\epsilon}}\right)^{E[W(X)]} ({Var}(X) +(\epsilon E[W(X)])^2)  \leq {Var}(X)\\
  \Longleftrightarrow \;& \left(\frac{e^\epsilon}{(1+\epsilon)^{1+\epsilon}}\right)^{E[W(X)]} (\epsilon E[W(X)])^2) \leq {Var}(X) \left(1-\left(\frac{e^\epsilon}{(1+\epsilon)^{1+\epsilon}}\right)^{E[W(X)]}\right)\\
\Longleftrightarrow \;& \frac{\left(\frac{e^\epsilon}{(1+\epsilon)^{1+\epsilon}}\right)^{E[W(X)]}(\epsilon E[W(X)])^2}{1-\left(\frac{e^\epsilon}{(1+\epsilon)^{1+\epsilon}}\right)^{E[W(X)]}} \leq   {Var}(X) \nonumber
\end{align*}
which shows our claim.
\end{proof}

\begin{figure}[t]
\centering
\includegraphics[width = 0.6\textwidth]{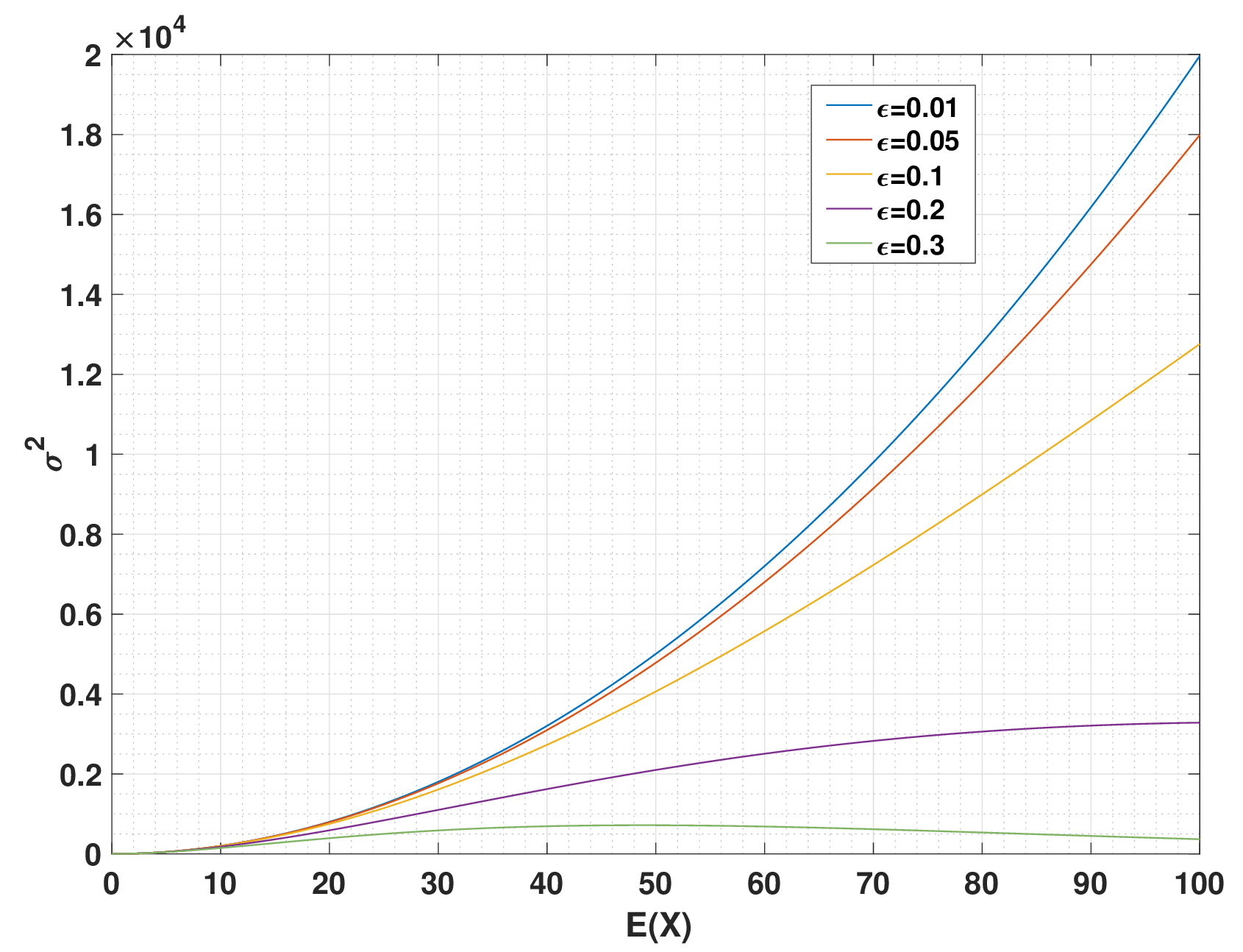}
\caption{The relationship between $E[W(X)]$ and $Var(X)$ based on different values of $\epsilon$.}
\label{fig:comp}
\end{figure}

We now further investigate the relationship between Chebyshev's inequality and the Chernoff bound. In Theorem~\ref{theo:compare}, the three parameters: $\epsilon, E[W(X)]$ and ${Var}(X)$ establish the relationship between Chebyshev's inequality and the Chernoff bound. Among the parameters, $\epsilon$ indicates the deviation from the expected value. After fixing the value of $\epsilon$, for any instance, the relationship between $E[W(X)]$ and ${Var}(X)$ can determine which inequality is more suitable for the purpose of solving the instance. As shown in Figure \ref{fig:comp}, the values of $\epsilon$ set independently at $\{0.01,0.05,0.1,0.2,0.3\}$. The figure is based on a test instance with $100$ items, and the weights of items are chosen uniformly at random in the interval $[0,1]$. Every curve in Figure~\ref{fig:comp} corresponds to a fixed value of $\epsilon$. When the tuple of $\left(E[W(X)], {Var}(X)\right)$  is located on the curve, the probability of the constraint violation calculated by Chebyshev's inequality and the Chernoff bound are the same. For the situation where a tuple of $\left(E[W(X)], {Var}(X)\right)$ is located above the curve, the Chernoff bound gives a better estimate, and if it is located below the curve, Chebyshev's inequality provides a better upper bound on the probability of the constraint violation. As can be seen from the figure, the greater the value of $\epsilon$, the more suitable the Chernoff bound is for obtaining a superior bound.

\section{Deterministic Approaches}
\label{sec:deter}

In this section, we consider several deterministic approaches to solving the chance-constrained knapsack problem. We present Integer Linear Programming (ILP), a heuristic approach based on Nemhauser Ullmann aalgorithm (NU-base) and dynamic programming (DP) in the following subsections.

\subsection{Integer Linear Programming }

Firstly, we linearise the surrogate functions which replace the chance constraint in the CCKP, then we apply Integer Linear Programming (ILP) to solve the problem. We consider Chebyshev's inequality and provide the linear approximation that characterises the ILP approach. The surrogate functions in Section \ref{Chebyshev's inequality} can be replaced by the following equations:

\begin{align}
  & \frac{\delta^2\sum_{i \in N} x_i}{\delta^2\sum_{i \in N} x_i+3(C- \sum_{i \in N} a_i x_i)^2}\leq \alpha \nonumber \\
  \Longleftrightarrow
  & \delta^2 \sum_{i \in N} x_i \leq \alpha \delta^2 \sum_{i \in N} x_i + 3\alpha \left(C-\sum_{i \in N} a_i x_i\right)^2 \nonumber\\
  \Longleftrightarrow
  & \frac{1-\alpha}{3\alpha}\delta^2 \sum_{i \in N} x_i \leq \left(C-\sum_{i \in N} a_i x_i\right)^2 \nonumber \\
 \Longleftrightarrow
 & \frac{1-\alpha}{3 \alpha} \delta^2\sum_{i \in N} x_i \leq C^2 - 2C\sum_{i \in N} a_ix_i + \left( \sum_{i \in N} a_ix_i\right)^2 \nonumber \\
 \Longleftrightarrow
 & \frac{1-\alpha}{3 \alpha} \delta^2\sum_{i \in N} x_i \leq C^2 - 2C\sum_{i \in N} a_ix_i + \sum_{i \in N} a_i^2 x_i +2 \sum_{i<j: i,j \in N} a_ia_j x_ix_j
 \label{cheby: unifor_delta ILP}
\end{align}
for the additive uniform distribution \eqref{cheby: uniform delta}. Then, we replace the term $x_ix_j$ in the right hand side in Equation \eqref{cheby: unifor_delta ILP} by a new added variable $y_{ij}$ with domain $\{0,1\}$ and define linear constraints for $y_{ij}$ as follows:
\begin{align}
    2y_{ij} \leq x_i +x_j \leq 1+y_{ij}.
    \label{con:addforYij}
\end{align}
The chance constraint can be reformulated as follows:
\begin{align}
    & \frac{1-\alpha}{3 \alpha} \delta^2\sum_{i \in N} x_i \leq C^2 - 2C\sum_{i \in N} a_ix_i + \sum_{i \in N} a_i^2 x_i +2 \sum_{i<j: i,j \in N} a_ia_j y_{ij} \nonumber\\
    \Longleftrightarrow
    &  \frac{1-\alpha}{3 \alpha} \delta^2\sum_{i \in N} x_i  + 2C\sum_{i \in N} a_ix_i - \sum_{i \in N} a_i^2 x_i -2 \sum_{i<j: i,j \in N} a_ia_j y_{ij} \leq C^2.
    \label{con:ILPyij}
\end{align}

For the chance-constrained knapsack problem, any feasible solution should subject to Equation \eqref{cheby: unifor_delta ILP} which indicate regardless of whether the value of $x_ix_j$ is equal to 0 or 1, setting $y_{ij}$ will not make the feasible solution infeasible. For example, if a solution $X$ is a feasible solution, it should submit to Equation \eqref{cheby: unifor_delta ILP}, assume there exist some $x_ix_j=0$ and we set the corresponding $y_{ij}$ to 1. So, we have $\sum_{i<j:i, j \in N}a_i a_j x_i x_j \leq \sum_{i<j:i,j \in N }a_i a_j y_{ij}$ of the solution $X$ which does not make the solution infeasible. Therefore, we can remove the right part of Equation \eqref{con:addforYij}. We then formulate the ILP model of the chance-constrained knapsack problem with weights conforming to the additive uniform distribution as follows:
\begin{align}
	\mathbf{max} \quad &\sum_{i \in N} p_ix_i\\
	\mathbf{s.t.}\quad &   \frac{1-\alpha}{3 \alpha} \delta^2\sum_{i \in N} x_i  + 2C\sum_{i \in N} a_ix_i - \sum_{i \in N} a_i^2 x_i -2 \sum_{i<j: i,j \in N} a_ia_j y_{ij} \leq C^2\\
	&  2y_{ij} \leq x_i +x_j 
	\label{con:ILP.additiveY} \\
	 & x_i, y_{ij}\in \{0,1\}^n
\end{align}
Similarly, for multiplicative uniform distribution \eqref{cheby: uniform beta}, we have
\begin{align}
 & \frac{\beta^2 \sum_{i \in N} {a_i}^2 x_i}{\beta^2 \sum_{i \in N} {a_i}^2 x_i+3(C- \sum_{i \in N} a_i x_i)^2} \leq \alpha \nonumber \\
 \Longleftrightarrow
 & \beta^2 \sum_{i \in N} a_i^2 x_i \leq \alpha \beta^2 \sum_{i \in N} x_i + 3\alpha \left(C-\sum_{i \in N} a_i x_i \right)^2 \nonumber \\
 \Longleftrightarrow
 & \frac{(1-\alpha)}{3\alpha}\beta^2 \sum_{i \in N} a_i^2 x_i \leq \left(C-\sum_{i \in N} a_i x_i \right)^2 \nonumber \\
 \Longleftrightarrow
 & \frac{(1-\alpha)}{3\alpha}\beta^2 \sum_{i \in N} a_i^2 x_i \leq C^2 - 2C\sum_{i=1}^n a_ix_i + \sum_{i \in N} a_i^2 x_i +2 \sum_{i<j: i,j \in N} a_ia_j x_ix_j
  \label{cheby: unifor_beta ILP}
\end{align}
and for the case that each weight of item is chosen according to the Normal distribution $N(a_i,\sigma^2_i)$ \eqref{cheb:normal:org}, we have
\begin{align}
  & \frac{\sum_{i \in N} \sigma_i^2 x_i}{\sum_{i \in N} \sigma_i^2 x_i+\left(C- \sum_{i \in N} a_i x_i\right)^2} \leq \alpha \nonumber\\
  \Longleftrightarrow
 & \sum_{i \in N}\sigma_i^2 x_i \leq \alpha \sum_{i \in N} \sigma_1^2 x_i +\alpha \left(C-\sum_{i \in N} a_i x_i \right)^2 \nonumber \\
  \Longleftrightarrow
  & \frac{(1-\alpha)}{\alpha}\sum_{i \in N} \sigma_i^2 x_i \leq \left(C-\sum_{i \in N} a_i x_i \right)^2 \nonumber \\
   \Longleftrightarrow
  & \frac{(1-\alpha)}{\alpha}\sum_{i \in N} \sigma_i^2 x_i \leq  C^2 - 2C\sum_{i \in N} a_ix_i + \sum_{i \in N} a_i^2 x_i +2 \sum_{i<j: i,j \in N} a_ia_j x_ix_j.
  \label{cheby: normal ILP}
\end{align}
Replacing the terms $x_ix_j$ in Equations (\ref{cheby: unifor_beta ILP}) and (\ref{cheby: normal ILP}) with the added variables $y_{ij}$, we can obtain the ILP models which take into account different surrogate constraints.

\subsection{Heuristic approach}

\begin{algorithm}[h]
\linespread{0.8}\selectfont
\caption{Heuristic approach for the CCKP}

\begin{algorithmic}[1]
\STATE \textbf{Input:} $n$ items with $\{a_1,..., a_n\}$ and $\{v_1,...,v_n\}$ which are the expected weights and the variances of the weights of items, and $\{p_1,...,p_n\}$ which are the profits of the items. Knapsack capacity $C$. 

\STATE let lists $L_1,..,L_n$.
\STATE Initialize $L_0= \{(0,0)\}$.
\FOR{$i = 1 $ to $ n$}
\STATE $L'_{i-1}=L_{i-1}$;
\FOR{each solution $K \in L'_{i-1}$}
\STATE add item $i$ to $K$ to generate a new solution $K'= K \cup \{i\}$,
\IF {solution $K'$ is feasible}{
\STATE replace $K$ with $K'$.
}
\ELSE 
\STATE delete $K$ from $L'_{i-1}$.
\ENDIF
\ENDFOR

\STATE let $P_{max}=-1$, $E=\{\}$
\WHILE{true}{
\STATE Let $k \in L_{i-1}$ be the first one with $p_k >P_{max}$,
\STATE let $k' \in L'_{i-1}$ be the first one with $p_{k'} >P_{max}$,
\IF{can not find $k$}{
\STATE insert remaining points from $L'_{i-1}$ into $E$,
\STATE \textbf{break}.
}
\ENDIF
\IF{can not find $k'$}{
\STATE insert remaining points from $L_{i-1}$ into $E$,
\STATE \textbf{break}.
}
\ENDIF

\IF{$S(k) < S(k')$ or $(S(k)= S(k') \  \text{and} \  p_k >p_{k'})$} {
\STATE insert $k$ into $E$, and set $P_{max}=p_k$,
}
\ELSE {
\STATE  insert $k'$ into $E$, and set $P_{max}=p_{k'}$,
}
\ENDIF
}
\ENDWHILE
\ENDFOR
\RETURN $L_n$.
\end{algorithmic}
\label{alg:NU}
\end{algorithm}

We now introduce a heuristic approach (see Algorithm \ref{alg:NU}) adapted from the Nemhauser-Ullmann algorithm (NU algorithm) proposed by Nemhauser and Ullmann \cite{Nemhauser}. The NU algorithm can be viewed as a sparse dynamic programming approach \cite{Beier04}, and it computers a list of all dominating sets in an iterative manner, adding one item after the other. For $i \in \{1,\ldots, n\}$, let $L_i$ denote the list of Pareto-optimal points with considering item 1 to $i$. The solutions in $L_i$ are assumed to be listed in increasing order of weight (profit). 

In the heuristic approach, we use surrogate weights obtained using Chebyshev's inequality and the Chernoff bound to replace the exact weights of the solution used to find the Pareto front in the NU algorithm. The heuristic approach starts with the empty set and then add items one by one until it finally obtains the set of Pareto-optimal packing for all $n$ items. $L_i$ can be computed using $L_{i-1}$ and the item $m$ as follows: first generate $L'_{i-1}$, add item $i$ to each element in $L_{i-1}$ if the new solution is feasible, and inset it to $L'_{i-1}$. Then, we merge the two lists $L_{i-1}$ and $L'_{i-1}$ according to the surrogate weight and the profit of the solutions. Finally, we obtain an order sequence $L_i$ of dominating solutions over the items $1,\ldots, i$. The resulting list $L_n$ contains all Pareto-optimal points for $n$ items. For this list, we choose the point with maximal profit, and the packing belonging to this result is the approximate optimal solution. 

In the merging step, both lists are sorted according to the surrogate weights. Thus, this task can be completely by scanning only once through both of these lists. However, this heuristic approach is effective in solving the knapsack problem where the weights of items are deterministic in value. In the chance-constrained knapsack problem, the weights of items are stochastic variables. Using the surrogate weights and the profits of solutions to find the Pareto front does not guarantee that the optimal solution to the problem will be found. Therefore, we introduce dynamic programming for the chance-constrained knapsack problem in the next subsection.

\subsection{Dynamic programming}

We now introduce a dynamic programming approach to solve the chance-constrained knapsack problem. Dynamic programming is one of the traditional approaches for the classical binary knapsack problems \cite{Toth}. In the DP approach, items are processed in the order according to their index, from $1$ to $n$.

The key idea behind the DP approach is to assume the weights of items are random variables with corresponding expected weights and variances. The approach is applied in a similar manner to the process which is undertaken for the classical binary knapsack problem. The program table of the chance-constrained knapsack problem consists of $n+1$ rows and $C+1$ columns which are used to compute the optimal solution. Each cell $M_{ij}$ consists of a set of feasible solutions which selected items from the items set $\{1,\ldots,i\}$ and knapsack capacity $j$, and those solutions are not dominated by each other. Here we choose to use the surrogate function according to the setting of the problem to tackle the chance constraint. To initialise the program table, we set the first cell as $M_{00} = \{\emptyset\}$ which only contains an empty set of items.

It can be observed of the surrogate functions obtained by Chebyshev's inequality \eqref{chebyshev} and Chernoff bound \eqref{chernoff:con} that for a solution, not only its expected weight but also its variance affect the probability with which it will violate the knapsack bound. With a fixed expected weight, when the value of the variance decrease, the probability of the chance constraint will decrease. Similarly, with a fixed variance, a decrease in the expected weight leads to a decrease in probability that the chance constraint will decrease. The smaller the value of this probability, the higher the capability to insert items into the knapsack. Therefore, we say that solution $S$ dominates solution $S'$, denoted by $S \succeq S'$, iff $w(S)<w(S'), v(S)<v(S')$ and $p(S)>p(S')$, where $w(S)$ denotes the expected weight of solution $S$, $v(S)$ denotes the variance of solution $S$ and $p(S)$ denotes the profit of solution $S$.

Let item $(i-1)$ be the predecessor of item $i$ and $a_i \leq j$. Based on the set of feasible solutions in $M_{(i-1)(j-a_i)}$ we compute $M_{ij}$ where $a_i$ denotes the expected weight of item $i$, and giving us $a_i \leq j$. We calculate $M_{ij}$ by adding item $i$ using the following steps. Firstly, all elements from $M_{(i-1)(j)}$ are copied to $M_{ij}$. Secondly, for every solution $S$ in $M_{(i-1)(j-a_i)}$, item $i$ is added to the set of items. If the new set of items $S\cup\{i\}$ is feasible and not dominated by any solution in $M_{ij}$, we remove solutions from $M_{ij}$ which are dominated by $S\cup\{i\}$ and add $S\cup\{i\}$ to $M_{ij}$. For the resulting algorithm (Algorithm \ref{alg:dp}), we can state the following theorem:

\begin{algorithm}[t]
\linespread{1.0}\selectfont
\caption{Dynamic programming for the CCKP}
\begin{algorithmic}[1]
\STATE \textbf{Input:} $n$ items with $\{a_1,..., a_n\}$ and $\{v_1,...,v_n\}$ which are the expected weights and the variances of the weights of items, and $\{p_1,...,p_n\}$ which are the profits of the items. Knapsack capacity $C$. 

\STATE let $M[0,\ldots,n][0,\ldots,C]$ be a new table, each cell $M_{ij}$ stores a set of solutions.
\STATE Initialize $M_{00}= \{\emptyset\}$.
\FOR{$i =1 $ to $ n$}
\FOR{ $j= 0 $ to $C$}
\STATE $M_{ij} = M_{(i-1)j}$;
\IF {$a_i\leq j$}
\FOR{each solution $S \in M_{(i-1)(j-a_i)}$}
\STATE add item $i$ to $S$ to generate a new solution $S'= S \cup \{i\}$,
\IF {solution $S'$ is feasible and not dominated by any solution in $M_{ij}$}
\STATE remove solutions in $M_{ij}$ which are dominated by $S'$.
\STATE add $S'$ to $M_{ij}$.
\ENDIF
\ENDFOR
\ENDIF
\ENDFOR
\STATE $M_{ij}$ stores all feasible and no-dominate to each other solutions.
\ENDFOR
\end{algorithmic}
\label{alg:dp}
\end{algorithm}

\begin{theorem}

For each set of item $\{1,..,i\}$, $M_{ij}$ stores a set of feasible solutions which are not dominated each other with considering all subset of $\{1,...,i\}$ and knapsack capacity $j$, and the optimal solution is among them. In particular, $M_{nC}$ contains the optimal solution with considering all items, which can be obtained via DP approach.

\end{theorem}

\begin{proof}

The statement is true for the first item as there are only two options: choosing or not choosing the first item. So $M_{00}$ stores the solution: $(\emptyset)$. Now we assume that $M_{(i-1)j}$ stores all feasible solutions which are not dominated by each other with respect to all subsets of $\{1,..,i-1\}$ with capacity $j$. 

Now, we contract $M_{ij}$ by first adding all subsets in $M_{(i-1)j}$. Then, if $a_i > j$, item $i$ is not able to add in any solution in $M_{(i-1)(j-a_i)}$ and $M_{ij}=M_{(i-1)j}$. Otherwise, adding item $i$ to a subset $S' \in M_{(i-1)(j-a_i)}$, if the new set of items is still feasible according to capacity $j$, then the expected weight of this solution is $w(S' \cup \{i\})= w(S')+a_i$, the variance of this solution is $v(S' \cup \{i\})= v(S')+\sigma_i^2$ and the profit is $p(S' \cup \{i\})= p(S')+p_i$. Then, if the new solution is not dominated by any solution in $M_{ij}$, we insert this solution into $M_{ij}$, removing all solutions in $M_{ij}$ which are dominated by this solution.  

Therefore, $M_{ij}$ stores all feasible solutions which are not dominated by each other, and we can pick the optimal solution in $M_{nC}$. This concludes the proof.

\end{proof}

We now investigate the runtime for this dynamic program. The feasible solutions in the cell $M_{ij}$ have been calculated by considering all possible combinations of variances and expected weights from the set $\{1,...,i\}$, in which $2^i$ denotes the number of different expected weights and variances in the worst case. We then give the upper bound of the computation time that DP takes to calculate all possible combinations of the variances and expected weights of $M_{ij}$ as $O(2^{2i})$. Therefore, the time until the DP approach calculates the optimal solution to the chance-constrained knapsack problem instance is the summary of $O(nC2^{2n})$. However, in the case that the weights of items conform to a uniform distribution and take values in $[a_i-\delta, a_i+\delta]$, then the variance of items are the same. So for the set $\{1,\ldots,i\}$, the possible combinations of the variance is $O(i)$, and the runtime of the instances, in this case, is bounded by $O(Cn^22^n)$. In the case that the weights of items conform to the uniform distribution $[(1-\beta)a_i, (1+\beta)a_i]$ or the Normal distribution $N(a_i, a_i\beta)$, the variance of an item is equal to the expected weight of this item time the uncertainty $\beta$. Therefore, when tacking those cases, the DP does not need to incorporate the variance of solutions into the domination comparison, and the runtime of the approach is bounded by $O(Cn2^n)$.

\section{Evolutionary algorithms for the CCKP}
\label{sec:EAs}

Evolutionary algorithms are bio-inspired randomized techniques, and they can obtain solutions with good quality in acceptable computation time. In this section, we discuss the use of evolutionary algorithms to solve the CCKP. We begin by designing standard fitness functions for a single-objective approach and a multi-objective approach. Next, we investigate the effectiveness of the simplest single-objective evolutionary algorithm (the (1+1)~EA) and its multi-objective version (GSEMO) to solve the problem using an experimental study. 

\subsection{Single-Objectives Approach} 

\begin{algorithm}[t]
\caption{(1+1)~EA}
\begin{algorithmic}[1]
\STATE Choose $x\in \{0,1\}^n$ uniformly at random.
\WHILE { \textit{stopping criterion not met}}
\STATE $y\leftarrow$ flip each bit of $x$ independently with probability of $\frac{1}{n}$;
\IF{$f(y)\geq f(x)$} 
\STATE $x \leftarrow y$ ;
\ENDIF
\ENDWHILE 
\end{algorithmic}
\label{oneone}
\end{algorithm}

We start by considering a single-objective approach and design a fitness function that can be used in single-objective evolutionary algorithms. The fitness function $\textit{f}$ for the approach needs to take into account the profit of the selected items and the requirement to meet the chance constraint. 
We define the fitness of a solution $X \in \{0,1\}^n$ as:
\begin{eqnarray}
f(X)=(u(X), v(X), P(X))
\label{fit1}
\end{eqnarray}
where $u(X)=max\{\sum_{i \in n}a_i x_i -c,0\}$, $v(X)=max\{Pr\{\sum_{i \in n} w_i x_i >c\}-\alpha,0\}$. For this fitness function, $u(X)$ and $v(X)$ need to be minimized while  $P(X)$ need to be maximized, and we optimize $\textit{f}$ in lexicographic order. The fitness function takes into account two types of infeasible solutions: (1) the expected weight of the solution exceeds the bound of capacity, (2) the probability that the total weight of the solution violating the capacity is greater than $\alpha$. Note that $\alpha$ is usually a small value, and throughout this paper, we work with $\alpha \leq 0.1$. The reason for having the first type of infeasible solutions is that we can not use the inequalities to guide the search, if the expected weight of a solution is below the given capacity bound. Furthermore, the first type of infeasible solutions violates the chance constraint with a probability at least $1/2$ for all probability distributions studied in the experimental part of this paper.

Among solutions that meet the chance constraint, we maximize the profit $P(X)$. Formally, we have the following relation between two search points $x$ and $y$:
\begin{eqnarray}
& & f(x) \succeq f(y) \nonumber\\
&\Longleftrightarrow & \left(u(x) < u(y)\right) or \left(u(x)=u(y)  \wedge v(x) < v(y)\right)\\
& & or \nonumber \left(u(x)=u(y) \wedge v(x)=v(y) \wedge P(x) > P(y) \right) \nonumber
\end{eqnarray}
When comparing solutions, a feasible solution is preferred in a comparison between any infeasible and this feasible solution. Between two feasible solutions, the one with better profit is preferred. When comparing two infeasible solutions, the one with a lower degree of constraint violation is better than the other. 

The fitness function can be used in any single-objective evolutionary algorithm. In this paper, we investigate the performance of the classical (1+1)~EA (see Algorithm~\ref{oneone}). We generate the initial solution with items chosen uniformly at random for the algorithm, and then the (1+1)~EA flips each bit of the current solution with a probability of $1/n$ in the mutation step. In the selection step, the algorithm accepts the offspring if it is at least as good as the parent according to the fitness function $\textit{f}$.

\subsection{Multi-Objective Approach}

Now we consider a multi-objective approach where each search point $X$ is a two-dimensional point in the objective space. We use the following fitness functions:
\begin{equation}
g_1(X)=\left\{
\begin{array}{lcl}
P_r(W(X)\geq C) & & {E[W(X)]<C}\\
1+(E[W(X)]-C)& & {E[W(X)]\geq C}
\label{g1x}
\end{array} \right.
\end{equation}

\begin{equation}
g_2(X)=\left\{
\begin{array}{lcl}
P(X) & & { g_1 (X)\leq \alpha}\\
-1 & & {g_1 (X) >\alpha}
\label{g2x}
\end{array} \right.
\end{equation}
where $W(X)$ denotes the weight of the solution $X$, $E[W(X)]$ denotes the expected weight of the solution. We say solution $Y$ weak-dominates solution $X$ denoted by $Y \succeq X$, \textit{iff} $g_1(Y)\leq g_1(X) \land g_2(Y) \geq g_2 (X)$. Comparing two solutions, the objective function ${g_1}$ guarantees that a feasible solution dominates all infeasible solutions. The objective function ${g_2}$ ensures that the search process is guided towards feasible solutions and that trade-offs in terms of confidence level and profit are computed for feasible solutions.

\begin{algorithm}[t]
\caption{GSEMO}
\begin{algorithmic}[1]
\STATE Choose $x \in \{0,1\}^n$ uniformly at random \;
\STATE $S\leftarrow \{x\}$;
\WHILE{stopping criterion not met}
\STATE choose $x\in S$ uniformly at random;
\STATE $y\leftarrow$ flip each bit of $x$ independently with probability of $\frac{1}{n}$;
\IF{($\not\exists w \in S: w \succeq_{GSEMO} y$)}
\STATE $S \leftarrow (S \cup \{y\})\backslash \{z\in S \vert y \succeq_{GSEMO} z\}$ ;
\ENDIF
\ENDWHILE
\end{algorithmic}
\label{alg:multiobj}
\end{algorithm}

The multi-objective algorithm we consider here is the Global Simple Evolutionary Multi-Objective Optimizer (GSEMO) (see Algorithm \ref{alg:multiobj}), which is a simple multi-objective evolutionary algorithm (SEMO) that searches globally. Laumanns et al., \cite{laumanns} investigated and analyzed a SEMO which starts with an initial solution and stores all non-dominated solutions in each population. In each step, it uniformly chooses a search point from the population and flips one bit of this search point to obtain a new search point. The new population contains search points corresponding to all non-dominated fitness vectors. Giel and Wegener \cite{Giel} investigated a GSEMO due to \cite{laumanns}, this GSEMO works like the SEMO but different in mutation step. In the mutation step, each bit of the search point is flipped independently of the others with probability $\frac{1}{n}$. When GSEMO applied to a single-objective optimization problem, it equals the (1+1)~EA, for both algorithms use the same mutation operator and keep at each time, and any solution does not dominate each other found so far in the optimization process.

\section{Experiments}
\label{sec:experimental}

In this section, we first compare the results obtained by using deterministic approaches, (1+1)~EA and GSEMO, and investigate the performance of these approaches with different surrogate functions of the chance constraint. We describe the experimental design and the chance-constrained knapsack problem instances in the next subsection.

\subsection{Experimental Setup}
\label{subsec:setup}

Firstly, we describe the experimental design and the chance-constrained knapsack problem instances.
In this chapter, all experiments were performed using Java (version 11.0.1) on a MacBook with a 2.3 GHz Intel Core i5 CPU. The benchmark we used is from the literature \cite{Roostapour}, which was created following the approach in \cite{Martello}. We choose two types of instances from the benchmark: \textit{uncorrelated} and \textit{bounded strongly correlated}. The weights and profits of items in the uncorrelated instances are integers that are chosen uniformly at random within $[1, 1000]$. The bounded strongly correlated instances have the tightest bound knapsack among all type of instances where the weights of items are chosen uniformly at random within $[1, 1000]$, and the values of profits are set by the weights. 

We adapt the above instances to the chance-constrained knapsack problem by randomising the weights. Since the weights of items must be positive, we add a value $\gamma$ to every deterministic weight from the benchmark and take it as the expected weight of this item to ensure we can factor in more uncertainty in all instances. Since we change the weights of items, we need to adjust the considered constraint bound. However, shifting the knapsack bound is challenging, as it is necessary to ensure that a solution remains feasible after bound has been changed. Moreover, increasing the knapsack's capacity expands the feasible search space and may introduce additional feasible solutions. Hence, when shifting the capacity of the knapsack, one should consider keeping the feasibility of original solutions and the size of the new feasible search space adaptive. 

We adjust the original knapsack problem instances from the benchmark set as follows. First, we sort the weights of items in ascending order. Then, the first $k$ items with smaller weights are chosen to be added to the knapsack until the original capacity is exceeded. Hence, this number of items $k$ represents the largest number of items that any feasible solution may include. We adapt the capacity bound according to this and set:

\begin{equation}
\begin{split}
C' \leftarrow C + k \gamma.\\
\end{split}
\end{equation}

We set $\gamma =100$ and apply it to the initial benchmark. In this section, we consider three instance categories: (1) instances in which every item weight conforms to the additive uniform distribution and takes the value in $[a_i -\delta, a_i +\delta]$; (2) instances in which every item weight conforms to the multiplicative uniform distribution and takes the value in $[(1-\beta)a_i, (1+\beta)a_i]$; (3) instances in which every item weight conforms to the Normal distribution $N\sim (a_i, a_i\beta)$. The tuples $(\alpha ,\delta, \beta)$ are the combinations of the elements from the sets $\alpha=[0.001,0.01,0.1]$, $\delta =[25, 50]$ and $\beta= [0.01,0.05,0.1]$. Based on this arrangement, we compare the performance of all the algorithms on the chance-constrained knapsack problem. Since (1+1)~EA and GSEMO are bio-inspired algorithms, they cannot produce exact optimal solutions, and the solutions are different in independent runs. Statistical comparisons (comparing each pair of algorithms with surrogate functions) are carried out using the Kruskal-Wallis test (with a $95\%$ confidence interval) integrated with the Wilcoxon sum rank test (with a $95\%$ confidence level). For more detailed descriptions of thees statistical tests, we refer the reader to \cite{Dunn,Driscoll,Ghasemi,corder2014nonparametric}. 

In the next subsection, we compare the performance of all proposed algorithms for instances of different types and sizes. Then, we highlight the differences between the algorithms using Chebyshev's inequality and the Chernoff bound as the surrogate functions of the chance constraint.

\subsection{Experimental Results}

We benchmark our approach with the combinations from the experimental setting described above. Table \ref{tab:100deltaCheby} and Table \ref{tab:100deltaChernoff} list the results obtained from the two types of instances which contain 100 items. The weights of items conform to an additive uniform distribution, and the best solutions among all approaches are emphasised in bold. For each instance, we investigate different settings together with different levels of uncertainty determined by $\delta$ and the requirement on the chance constraint determined by $\alpha$. We apply Chebyshev's inequality to ILP by fixing running time $\{2mins, 6mins, 10mins\}$ for all instances; the results are listed in Table \ref{tab:100deltaCheby}. We use both chance-constrained estimation methods to tackle the chance constraint when using the heuristic approach and the DP approaches. Under the heuristic approach and the DP headings, the \textit{profit} denotes the object value of each approach, and the \textit{time(ms)} denotes the computation time associated with each approach. The computation time associate with DP is one or orders of magnitude longer than for other approaches, for all instances. Where units are presented in double inverted commas ($"-"$), this means that the run time for DP to solve these instances is longer than ten hours. For (1+1)~EA and GSEMO, we provide the results from $30$ independent runs with $10^6$ generations for all instances. I these cases, \textit{profit} denotes the average profit associated with the $30$ runs, and \textit{time(ms)} denotes the average running time associated with the $30$ runs.

\begin{table}[t]
  \centering
  \caption{Statistical results for instances with 100 items based on additive uniform distribution and using Chebyshev's inequality}
  \scalebox{0.7}{
  \makebox[\linewidth][c]{
  \tabcolsep=0.1cm
    \begin{tabular}{crrrrrrrrrrrrrrrrrrr}
    \toprule
          & \multicolumn{1}{l}{Capacity} & \multicolumn{1}{l}{$\delta$} & \multicolumn{1}{l}{$\alpha$} &       & \multicolumn{3}{c}{ILP}             &       & \multicolumn{2}{c}{Heuristic approach}       &       & \multicolumn{2}{c}{DP}        &       & \multicolumn{2}{c}{(1+1)EA}        &       & \multicolumn{2}{c}{GSEMO}  \\
          &       &       &       &       & \multicolumn{1}{l}{2mins} & \multicolumn{1}{l}{6mins} & \multicolumn{1}{l}{10mins} &       & \multicolumn{1}{l}{profit} & \multicolumn{1}{l}{time(ms)} &       & \multicolumn{1}{l}{profit} & \multicolumn{1}{l}{time(ms)} &       & \multicolumn{1}{l}{profit} & \multicolumn{1}{l}{time(ms)} &       & \multicolumn{1}{l}{profits} & \multicolumn{1}{l}{time(ms)} \\
          \midrule
         bou-s-c 1  & 11775 & 25    & 0.001 &       & 13701 & 13701 & 13701 &       & $\mathbf{13707}$ & 261   &       & $-$ & $-$ &       & 13614.8 & 1200.521 &       & $\mathbf{13707}$ & 21090.318 \\
          &       &       & 0.01  &       &$\mathbf{15252}$  & $\mathbf{15252}$ & $\mathbf{15252}$ &       & $\mathbf{15252}$ & 303   &       & $-$ & $-$ &       & 15150.47 & 1207.292 &       & $\mathbf{15252}$ & 19888.28 \\
          &       &       & 0.1   &       & 15768 & 15775 & 15775 &       & $\mathbf{15782}$ & 492   &       & $-$ & $-$ &       & 15680.87 & 1206.771 &       & $\mathbf{15782}$ & 26239.498 \\
          &       & 50    & 0.001 &       & 11757 & 11793 & 11795 &       & $\mathbf{11900}$ & 517  &   &  $-$  & $-$ &       & 11756.27 & 1195.313 &       & 11888.1 & 11637.863 \\
          &       &       & 0.01  &       & 14503 & $\mathbf{14505}$  & $\mathbf{14505}$  &       & $\mathbf{14505}$  & 691   &       & $-$  & $-$ &       & 14416.8 & 1203.125 &       & $\mathbf{14505}$  & 16181.898 \\
          &       &       & 0.1   &       & $\mathbf{15585}$   & $\mathbf{15585}$  & $\mathbf{15585}$  &       & $\mathbf{15585}$  & 304   &       & $-$  & $-$ &       & 15431.37 & 1208.854 &       & $\mathbf{15585}$  & 18520.158 \\
          \hline
        bou-s-c 2  & 31027 & 25    & 0.001 &       & 35045 & 35045 & 35045 &       &  $\mathbf{35069}$ & 2015  &       & $-$  & $-$ &       & 34874.8 & 1226.042 &       & 35068.933 & 51648.881 \\
          &       &       & 0.01  &       & 37005 & 37005 & 37005 &       &   $\mathbf{37027}$  & 2287  &       &   $-$  &  $-$  &       & 36850.73 & 1232.813 &       & 37019.133 & 61361.251 \\
          &       &       & 0.1   &       & 37647 & 37647 & 37657 &       &  $\mathbf{37673}$   & 5056  &       &  $-$ &  $-$  &       & 37467.43 & 1233.854 &       & 37670.367 & 64439.556 \\
          &       & 50    & 0.001 &       & 32096 & 32270 & 32357 &       & 32547 & 2092  &       & $-$  &  $-$  &       & 32332.97 & 1220.833 &       &   $\mathbf{32552.733}$ & 42825.912 \\
          &       &       & 0.01  &       & 36019 & 36033 & 36061 &       & $\mathbf{36131}$ & 2172  &       & $-$  &  $-$  &       & 35937.9 & 1228.125 &       & 36121.667 & 58301.656 \\
          &       &       & 0.1   &       & 37391 & 37391 & 37391 &       & $\mathbf{37406}$ & 2857  &       &  $-$ &  $-$  &       & 37202.63 & 1235.938 &       & 37370.9 & 62496.325 \\
          \hline
        bou-s-c 3   & 58455 & 25    & 0.001 &       & 64190 & 64175 & 64265 &       & $\mathbf{64389}$ & 8042  &     &    $-$     &     $-$    &       & 64185.73 & 1250  &       & 64387.067 & 203302.72 \\
          &       &       & 0.01  &       & 66630 & 66630 & 66630 &       & $\mathbf{66641}$ & 10749 &      &    $-$     &   $-$      &       & 66404.37 & 1253.125 &       & 66639.9 & 154383.75 \\
          &       &       & 0.1   &       & 67339 & 67339 & 67339 &       & $\mathbf{67357}$ & 31192 &        &    $-$     &     $-$    &       & 67164.5 & 1256.771 &       & 67356.733 & 162281.6 \\
          &       & 50    & 0.001 &       & 61111 & 61111 & 61111 &       & 61155 & 8374  &       &      $-$   &    $-$     &       & 61007.6 & 1242.709 &       & $\mathbf{ 61220.2}$ & 168099.53 \\
          &       &       & 0.01  &       & 65478 & 65496 & 65491 &       & $\mathbf{65603}$ & 8420  &       &   $-$      &     $-$    &       & 65372.4 & 1250.521 &       & 65601.8 & 200503.73 \\
          &       &       & 0.1   &       & 66953 & 67001 & 66953 &       & $\mathbf{67059}$ & 17751 &       &    $-$     &   $-$      &       & 66859.37 & 1254.688 &       & 67057.3 & 213127.4 \\
          \hline
   uncorr 1 & 7715  & 25    & 0.001 &       & $\mathbf{ 14354}$ & $\mathbf{ 14354}$ & $\mathbf{ 14354}$ &       & $\mathbf{ 14354}$ & 128   &       & $\mathbf{ 14354}$ & 12605162 &       & 14273.6 & 1205.729 &       & $\mathbf{ 14354}$ & 11479.932 \\
          &       &       & 0.01  &       & $\mathbf{ 16481}$ & $\mathbf{ 16481}$ & $\mathbf{ 16481}$ &       & $\mathbf{ 16481}$ & 120   &       & $\mathbf{ 16481}$ & 23461453 &       & 16433.1 & 1213.021 &       & $\mathbf{ 16481}$ & 13291.141 \\
          &       &       & 0.1   &       & $\mathbf{ 17247}$  & $\mathbf{ 17247}$  & $\mathbf{ 17247}$  &       & $\mathbf{ 17247}$ & 153   &       & $\mathbf{ 17247}$ & 49890459 &       & 17176.57 & 1213.021 &       & $\mathbf{ 17247}$ & 13480.268 \\
          &       & 50    & 0.001 &       & $\mathbf{11599}$ & $\mathbf{11599}$ & $\mathbf{11599}$ &       & $\mathbf{11599}$ & 53    &       & $\mathbf{11599}$ & 33147184 &       & 11478.83 & 1196.875 &       & $\mathbf{11599}$ & 11852.237 \\
          &       &       & 0.01  &       & $\mathbf{15504}$ & $\mathbf{15504}$ & $\mathbf{15504}$  &       & $\mathbf{15504}$  & 188   &       & $\mathbf{15504}$  & 205023878 &       & 15424.1 & 1207.292 &       & $\mathbf{15504}$  & 16188.438 \\
          &       &       & 0.1   &       & $\mathbf{16890}$  & $\mathbf{16890}$ & $\mathbf{16890}$   &       & $\mathbf{16890}$   & 134   &       & $\mathbf{16890}$ &  36045784     &       & 16814.53 & 1214.583 &       & $\mathbf{16890}$ & 14843.431 \\
          \hline
         uncorr 2 & 19545 & 25    & 0.001 &       & 27027 & 27027 & 27027 &       & $\mathbf{27029}$ & 601   &       &  $-$   &     $-$    &       & 26932.67 & 1232.292 &       & $\mathbf{27029}$ & 38636.347 \\
          &       &       & 0.01  &       & 28786 & $\mathbf{28825}$ & $\mathbf{28825}$ &       & $\mathbf{28825}$ & 879   &       &  $-$  &    $-$     &       & 28724.13 & 1238.021 &       & $\mathbf{28825}$ & 54770.144 \\
          &       &       & 0.1   &       & $\mathbf{29415}$ & $\mathbf{29415}$ & $\mathbf{29415}$ &       & $\mathbf{29415}$ & 853   &       &  $-$  &   $-$      &       & 29315.6 & 1241.667 &       & $\mathbf{29415}$ & 56148.136 \\
          &       & 50    & 0.001 &       & 24561 & 24561 & 24561 &       & $\mathbf{24565}$ & 625   &       &   $-$  &   $-$      &       & 24449.3 & 1228.125 &       & $\mathbf{24565}$ & 32166.648 \\
          &       &       & 0.01  &       & 27962 & 27962 & 27962 &       & $\mathbf{27985 }$ & 504   &       &  $-$  &     $-$    &       & 27918.43 & 1237.5 &       & $\mathbf{27985}$  & 40521.42 \\
          &       &       & 0.1   &       & $\mathbf{29165 }$  & $\mathbf{29165 }$ & $\mathbf{29165 }$ &       & $\mathbf{29165 }$ & 648   &       & $-$   &    $-$     &       & 29091.97 & 1238.542 &       & $\mathbf{29165 }$ & 43018.24 \\
          \hline
         uncorr 3 & 36091 & 25    & 0.001 &       & 39181 & 39182 & 39182 &       & $\mathbf{39245 }$ & 1150  &       &  $-$  & $-$ &       & 39192.53 & 1258.333 &       & $\mathbf{39245 }$ & 32204.292 \\
          &       &       & 0.01  &       & $\mathbf{40581 }$ & $\mathbf{40581 }$ & $\mathbf{40581 }$ &       & $\mathbf{40581 }$ & 998   &       & $-$   &     $-$    &       & 40530.27 & 1261.458 &       & $\mathbf{40581 }$ & 46855.019 \\
          &       &       & 0.1   &       &$\mathbf{40991 }$  & $\mathbf{40991 }$  & $\mathbf{40991 }$  &       & $\mathbf{40991 }$  & 1300  &       &  $-$  &     $-$    &       & 40890.8 & 1262.5 &       & 40990.8 & 44955.924 \\
          &       & 50    & 0.001 &       & 36531 & 37068 & 37098 &       & $\mathbf{37180 }$  & 1299  &       & $-$  &   $-$      &       & 37120.03 & 1256.25 &       & $\mathbf{37180 }$  & 45597.177 \\
          &       &       & 0.01  &       & 39739 & 39917 & 39917 &       & $\mathbf{39960 }$  & 1186  &       &   $-$ &    $-$     &       & 39906.7 & 1260.938 &       & $\mathbf{39960 }$  & 34677.722 \\
          &       &       & 0.1   &       & 40754 & $\mathbf{40811 }$  & $\mathbf{40811 }$ &       & $\mathbf{40811 }$ & 1315  &       &  $-$  &   $-$      &       & 40751.47 & 1263.542 &       & $\mathbf{40811 }$ & 38348.789 \\
          \bottomrule
    \end{tabular}}}%
   \label{tab:100deltaCheby}%
\end{table}%

\begin{table}[t]
  \centering
  \caption{Statistical results for instances with 100 items based on additive uniform distribution and using Chernoff bound}
  \scalebox{0.75}{
  \makebox[\linewidth][c]{
  \tabcolsep=0.1cm
    \begin{tabular}{crrrrrrrrrrrrrrr}
    \toprule
          & \multicolumn{1}{l}{Capacity} & \multicolumn{1}{l}{$\delta$} & \multicolumn{1}{l}{$\alpha$} &       & \multicolumn{2}{c}{Heuristic approach}        &       & \multicolumn{2}{c}{DP}        &       & \multicolumn{2}{c}{(1+1)EA}        &       & \multicolumn{2}{c}{GSEMO}  \\
          &       &       &       &       & \multicolumn{1}{l}{profit} & \multicolumn{1}{l}{time(ms)} &       & \multicolumn{1}{l}{profit} & \multicolumn{1}{l}{time(ms)} &       & \multicolumn{1}{l}{profit} & \multicolumn{1}{l}{time(ms)} &       & \multicolumn{1}{l}{profits} & \multicolumn{1}{l}{time(ms)} \\
          \midrule
          bou-s-c 1 & 11775 & 25    & 0.001 &       & $\mathbf{15208}$  & 293   &       & $-$  & $-$  &       & 15104.57 & 1236.979 &       & $\mathbf{15208}$  & 20055.67 \\
          &       &       & 0.01  &       & $\mathbf{15348}$  & 613   &       & $-$  & $-$  &       & 15232.07 & 1236.458 &       & $\mathbf{15348}$ & 25803.75 \\
          &       &       & 0.1   &       & $\mathbf{15599}$  & 1116  &       & $-$  & $-$  &       & 15454.03 & 1231.25 &       & $\mathbf{15599}$ & 20023.96 \\
          &       & 50    & 0.001 &       & 14367  & 715   &       & $-$  & $-$  &       & 14298.13 & 1243.75 &       & $\mathbf{14406}$ & 23056.49 \\
          &       &       & 0.01  &       & $\mathbf{14742}$  & 556   &       &$-$  & $-$  &       & 14613.8 & 1240.104 &       & $\mathbf{14742}$  & 16970.46 \\
          &       &       & 0.1   &       & $\mathbf{15144}$  & 317   &       & $-$   & $-$  &       & 15011.27 & 1235.417 &       & $\mathbf{15144}$  & 17916.35 \\
          \hline
       bou-s-c 2   & 31027 & 25    & 0.001 &       & 36893 & 2545  &       & $-$    & $-$  &       & 36824.33 & 1264.063 &       & $\mathbf{37027}$  & 42719.51 \\
          &       &       & 0.01  &       & $\mathbf{37219}$   & 9410  &       & $-$   & $-$  &       & 36980.27 & 1232.012 &       & $\mathbf{37219}$  & 44305.01 \\
          &       &       & 0.1   &       & $\mathbf{37437}$   & 4156  &       & $-$  & $-$  &       & 37241.03 & 1258.854 &       & $\mathbf{37437}$ & 48447.63 \\
          &       & 50    & 0.001 &       & $\mathbf{36071}$ & 1910  &       &  $-$  & $-$  &       & 35870.83 & 1268.229 &       & 36060.73 & 55605.18 \\
          &       &       & 0.01  &       &  $\mathbf{36423}$  & 2334  &       & $-$    & $-$  &       & 36250.57 & 1226.667 &       & 36416.27 & 60741.89 \\
          &       &       & 0.1   &       & $\mathbf{36893}$  & 3088  &       & $-$   & $-$  &       & 36712.6 & 1346.067 &       & 36888.93 & 57344.08 \\
          \hline
        bou-s-c 3  & 58455 & 25    & 0.001 &       & 65983 & 14478 &       &   $-$     &     $-$   &       & 66404.37 & 1289.063 &       & $\mathbf{66635}$   & 174822.8 \\
          &       &       & 0.01  &       & 66175 & 13885 &       &   $-$     &    $-$    &       & 66612.5 & 1286.458 &       & $\mathbf{66840.87}$   & 119783.3 \\
          &       &       & 0.1   &       & 66407 & 6809  &       &    $-$    &    $-$    &       & 66853.47 & 1286.979 &       & $\mathbf{67095}$ & 123643.2 \\
          &       & 50    & 0.001 &       & 64972 & 20041 &       &   $-$     &    $-$    &       & 65304.87 & 1292.708 &       &$\mathbf{65552.8}$   & 162435.2 \\
          &       &       & 0.01  &       & 65355 & 10396 &       &    $-$    &     $-$   &       & 65755.47 & 1292.188 &       &$\mathbf{65952.93}$  & 165467.6 \\
          &       &       & 0.1   &       & 65879 & 11290 &       &   $-$     &   $-$     &       & 66270.53 & 1291.667 &       & $\mathbf{66451}$   & 174008.9 \\
          \hline
          uncorr 1 & 7715  & 25    & 0.001 &       & $\mathbf{16432}$  & 129   &       & $\mathbf{16432}$ & 169791803  &       & 16342.87 & 1255.729 &       & $\mathbf{16432}$ & 12182.77 \\
          &       &       & 0.01  &       & $\mathbf{16638}$ & 136   &       & $\mathbf{16638}$ & 194176080  &       & 16575.9 & 1240.625 &       & $\mathbf{16638}$ & 11075.88 \\
          &       &       & 0.1   &       & $\mathbf{16945}$  & 144   &       & $\mathbf{16945}$  & 216867406 &       & 16833.13 & 1235.938 &       & $\mathbf{16945}$  & 10966.21 \\
          &       & 50    & 0.001 &       &$\mathbf{15299}$   & 98    &       & $\mathbf{15299}$ & 150530586  &       & 15243.03 & 1254.688 &       & $\mathbf{15299}$ & 15636.04 \\
          &       &       & 0.01  &       & $\mathbf{15812}$  & 161   &       & $\mathbf{15812}$  & 161619933  &       & 15711.03 & 1253.125 &       & $\mathbf{15812}$  & 16543.41 \\
          &       &       & 0.1   &       & $\mathbf{16362}$   & 142   &       & $\mathbf{16362}$  & 174426689 &       & 16287.27 & 1245.313 &       & $\mathbf{16362}$  & 18197.93 \\
          \hline
          uncorr 2 & 19545 & 25    & 0.001 &       & $\mathbf{28786}$  & 634   &       &$-$     & $-$    &       & 28690.93 & 1275  &       & $\mathbf{28786}$  & 29056.24 \\
          &       &       & 0.01  &       &$\mathbf{28982}$   & 839   &       & $-$      & $-$    &       & 28864.17 & 1271.354 &       & $\mathbf{28982}$   & 37327.55 \\
          &       &       & 0.1   &       & $\mathbf{29165}$   & 496   &       & $-$      & $-$    &       & 29090.6 & 1267.708 &       & $\mathbf{29165}$   & 38588.23 \\
          &       & 50    & 0.001 &       & $\mathbf{27962}$   & 469   &       & $-$     & $-$    &       & 27824.13 & 1279.688 &       & $\mathbf{27962}$  & 37555.58 \\
          &       &       & 0.01  &       &$\mathbf{28276}$   & 793   &       & $-$    & $-$    &       & 28165.67 & 1279.167 &       & $\mathbf{28276}$ & 38640.84 \\
          &       &       & 0.1   &       &$\mathbf{28732}$  & 823   &       & $-$     & $-$    &       & 28618.13 & 1276.042 &       & $\mathbf{28732}$  & 39601.95 \\
          \hline
         uncorr 3  & 36091 & 25    & 0.001 &       &$\mathbf{40581}$   & 935   &       & $-$   & $-$    &       & 40539.3 & 1305.729 &       & $\mathbf{40581}$ & 17069.47 \\
          &       &       & 0.01  &       & $\mathbf{40663}$ & 983   &       & $-$    & $-$    &       & 40639.67 & 1303.646 &       & $\mathbf{40663}$ & 24246.41 \\
          &       &       & 0.1   &       & $\mathbf{40830}$ & 1439  &       & $-$    & $-$    &       & 40741.6 & 1300.521 &       & $\mathbf{40830}$ & 40830 \\
          &       & 50    & 0.001 &       & $\mathbf{39960}$  & 1330  &       & $-$  & $-$    &       & 39894.23 & 1307.292 &       & $\mathbf{39960}$ & 51009.26 \\
          &       &       & 0.01  &       & $\mathbf{40150}$  & 1505  &       & $-$    & $-$    &       & 40123.97 & 1306.25 &       & 40149.4 & 40985.79 \\
          &       &       & 0.1   &       & $\mathbf{40482}$  & 1246  &       & $-$   
          & $-$    &       & 40422.3 & 1306.771 &   & 40481.67 & 39895.53 \\
          \bottomrule
    \end{tabular}}}%
  \label{tab:100deltaChernoff}%
\end{table}%

\begin{table}[t]
  \centering
  \caption{Statistical results for instances with 500 items based on additive uniform distribution}
  \scalebox{0.4}{
 \makebox[\linewidth][c]{
 \tabcolsep=0.1cm
    \begin{tabular}{crrrrrrrrrrrrrrrrrrrrr}
    \toprule
          & \multicolumn{1}{l}{Capacity} & \multicolumn{1}{l}{$\delta$} & \multicolumn{1}{l}{$\alpha$} &       & \multicolumn{3}{c}{ILP(Chebyshev)} &       & \multicolumn{6}{c}{(1+1)EA}                   &       & \multicolumn{6}{c}{GSEMO} \\
          &       &       &       &       &       &       &       &        & \multicolumn{3}{c}{Chernoff(1)} & \multicolumn{3}{c}{Chebyshev(2)} &       & \multicolumn{3}{c}{Chernoff(3)} & \multicolumn{3}{c}{Chebyshev(4)} \\
          &       &       &       &       & \multicolumn{1}{l}{2mins} & \multicolumn{1}{l}{6mins} & \multicolumn{1}{l}{10mins}    &   & \multicolumn{1}{l}{profit} & \multicolumn{1}{l}{time(ms)} & \multicolumn{1}{l}{stat} & \multicolumn{1}{l}{profit} & \multicolumn{1}{l}{time(ms)} & \multicolumn{1}{l}{stat} &       & \multicolumn{1}{l}{profit} & \multicolumn{1}{l}{time(ms)} & \multicolumn{1}{l}{stat} & \multicolumn{1}{l}{profit} & \multicolumn{1}{l}{time(ms)} & \multicolumn{1}{l}{stat} \\
          \midrule 
    bou-s-c 1 & 61447 & 25    & 0.001 &       & 74403 & 74385 & 74385 &       & 76464.3 & 5887.5 & {2(+),3(-),4(+)} & 73213.8 & 5866.67 &  {1(-),3(-),4(-)} &       & $\mathbf{77750.4}$   & 25626.04 &  {1(+),2(+),4(+)} & 74207.1 & 27694.79 &  {1(-),2(+),3(-)} \\
          &       &       & 0.01  &       & 77951 & 77951 & 77951 &       & 76890.33 & 5887.50 & 2(+),3(-),4(-) & 76569.53 & 5879.17 & 1(-),3(-),4(-) &       & $\mathbf{78064.47}$   & 25572.4 & 1(+),2(+),4(+) & 77552.07 & 28609.9 & 1(+),2(+),3(-) \\
          &       &       & 0.1   &       & 78649 &$\mathbf{78912}$  & $\mathbf{78912}$ &            &   77194.97 & 5889.06 & 2(-),3(-),4(-) & 77642.27 & 5883.33 & 1(+),3(-),4(-) &       & 78429.77 & 25743.23 & 1(+),2(+),4(-) & 78672.67 & 29402.6 & 1(+),2(+),3(+)   \\
          &       & 50    & 0.001 &       & 66465 & 69443 & 69443 &         &   74964.07 & 5894.79 & 2(+),3(-),4(+) & 68709.77 & 5851.04 & 1(-),3(-),4(-) &       &$\mathbf{76008.8}$  & 28145.31 & 1(+),2(+),4(+) & 69627.7 & 25955.73 & 1(-),2(+),3(-)  \\
          &       &       & 0.01  &       & 76121 & 76222 & 76222 &        & 75627.37 & 5903.13 & 2(+),3(-),4(-) & 75038.43 & 5869.27 & 1(-),3(-),4(-) &       & $\mathbf{76629.73}$  & 28291.67 & 1(+),2(+),4(+) & 75965.73 & 27930.73 & 1(+),2(+),3(-)\\
          &       &       & 0.1   &       & 46027 &$\mathbf{78227}$  & $\mathbf{78227}$ &       &76319.97 & 5892.71 & 2(-),3(-),4(-) & 77231.63 & 5878.65 & 1(+),3(-),4(-) &       & 77385.53 & 28717.19 & 1(+),2(+),4(-) &  78187.27& 28957.81 & 1(+),2(+),3(+)\\
          \hline
        bou-s-c 2  & 162943 & 25    & 0.001 &       & 177096 & 186590 & 186590 &    &    188967.4 & 6010.94 & 2(+),3(-),4(+) & 184846.9 & 5990.10 & 1(-),3(-),4(-) &       &$\mathbf{190575.8}$ & 31282.81 & 1(+),2(+),4(+) & 186081.5 & 34696.35 & 1(-),2(+),3(-)\\
          &  &       & 0.01  &       & 190681 & 190681 & 190757 &      & 189348.5 & 6015.10 & 2(+),3(-),4(-) & 189003.4 & 6017.71 & 1(-),3(-),4(-) &       & $\mathbf{190940.3}$ & 31718.23 & 1(+),2(+),4(+) & 190297.6 & 36672.92 & 1(+),2(+),3(-) \\
          &  &       & 0.1   &       &$\mathbf{192116}$  & $\mathbf{192116}$ & $\mathbf{192116}$ &     &    189852.2 & 6002.08 & 2(-),3(-),4(-) & 190439.4 & 6009.38 & 1(+),3(-),4(-) &       & 191380.6 & 31677.6 & 1(+),2(+),4(-) &  191632 & 37396.88 & 1(+),2(+),3(+)  \\
          &       & 50    & 0.001 &       & 180281 & 180348 & 180348 &       & 187149.1 & 6005.21 & 2(+),3(-),4(+) & 178807.7 & 5981.77 & 1(-),3(-),4(-) &       & $\mathbf{188326.5}$& 36233.33 & 1(+),2(+),4(+) & 180210 & 32353.13 & 1(-),2(+),3(-)  \\
          &       &       & 0.01  &       & 187042 & 187079 & 187079 &     &    187931.4 & 6008.85 & 2(+),3(-),4(-) & 187005.2 & 5997.40 & 1(-),3(-),4(-) &       &$\mathbf{189064.4}$ & 36852.6 & 1(+),2(+),4(+) & 188299.7 & 35451.56 & 1(+),2(+),3(-)\\
          &       &       & 0.1   &       & 189477 & 190630 & 190630 &       &  188798.8 & 6011.46 & 2(-),3(-),4(-) & 189713 & 6005.21 & 1(+),3(-),4(-) &       & 189986 & 37560.42 & 1(+),2(+),4(-) & $\mathbf{190993.3}$ & 37417.19 & 1(+),2(+),3(+) \\
          \hline
       bou-s-c 3   & 307286 & 25    & 0.001 &       & 341305 & 341305 & 341453 &       &   344059.7 & 6132.81 & 2(+),3(-),4(+) & 339020.5 & 6126.56 & 1(-),3(-),4(-) &       & $\mathbf{346115.3}$ & 38438.54 & 1(+),2(+),4(+) & 340895.1 & 40478.65 & 1(-),2(+),3(-)  \\
          & &       & 0.01  &       & 346271 & 346271 & 346271 &        &   344460 & 6126.56 & 2(+),3(-),4(-) & 343891 & 6135.94 & 1(-),3(-),4(-) &       & $\mathbf{346523.6}$ & 39475 & 1(+),2(+),4(+) & 345815.1 & 43563.54 & 1(+),2(+),3(-)  \\
          &  &       & 0.1   &       & 347725 & 347729 & $\mathbf{347905}$ &       &     344976.1 & 6166.56 & 2(-),3(-),4(-) & 345622.4 & 6134.38 & 1(+),3(-),4(-) &       & 347048.1 & 39237.5 & 1(+),2(+),4(-) &  347418.7 & 44185.42 & 1(+),2(+),3(+)  \\
          &       & 50    & 0.001 &       & 169122 & 169122 & 210932 &     &     341860.8 & 6127.08 & 2(+),3(-),4(+) & 332022.1 & 6108.85 & 1(-),3(-),4(-) &       & $\mathbf{343418.6}$ & 47959.9 & 1(+),2(+),4(+) & 333882.8 & 38178.13 & 1(-),2(+),3(-) \\
          &       &       & 0.01  &       & 341530 & 343866 & 344052 &    &     342619.9 & 6132.29 & 2(+),3(-),4(-) & 341661.9 & 6108.85 & 1(-),3(-),4(-) &       & $\mathbf{344242} $ & 48193.75 & 1(+),2(+),4(+) & 343537.7 & 41935.42 & 1(+),2(+),3(-)  \\
          &       &       & 0.1   &       & 346321 & 347064 &$\mathbf {347064}$ &         &     343776 & 6128.65 & 2(-),3(-),4(-) & 344848.8 & 6131.77 & 1(+),3(-),4(-) &       & 345325.3 & 49109.38 & 1(+),2(+),4(-) & 346723.6 & 44240.1 & 1(+),2(+),3(+)  \\
          \hline
          uncorr 1 & 37686 & 25    & 0.001 &       & 16998 & 58826 & 66238 &    &  85264.77 & 5956.25 & 2(+),3(-),4(+) & 80509.37 & 5919.79 & 1(-),3(-),4(-) &       & $\mathbf{86117.7}$  & 20050.52 & 1(+),2(+),4(+) & 81319.77 & 20571.88 & 1(-),2(+),3(-)  \\
          &  &       & 0.01  &       & 86128 & 86128 & 86128 &      &  85685.3 & 5955.73 & 2(+),3(-),4(-) & 85187.5 & 5930.73 & 1(-),3(-),4(-) &       & $\mathbf{86540.37}$ & 20035.94 & 1(+),2(+),4(+) & 86002.63 & 21854.17 & 1(+),2(+),3(-)  \\
          &  &       & 0.1   &       & 87252 & 87675 &   $\mathbf{87675}$      &         &   86218.6 & 5950.00 & 2(-),3(-),4(-) & 86899.67 & 5939.06 & 1(+),3(-),4(-) &       & 87066.83 & 20544.27 & 1(+),2(+),4(-) &  87518.63 & 22421.88 & 1(+),2(+),3(+)  \\
          &       & 50    & 0.001 &       &   15712    &   59369    &   70672    &       &   83116.43 & 5955.73 & 2(+),3(-),4(+) & 73929.1 & 5895.31 & 1(-),3(-),4(-) &       & $\mathbf{83829.9}$  & 23189.06 & 1(+),2(+),4(+) & 74731.03 & 20529.69 & 1(-),2(+),3(-) \\
          &       &       & 0.01  &       & 58826      &    83985   &   83985    &      &     84020.67 & 5955.21 & 2(+),3(-),4(+) & 83065.43 & 5917.19 & 1(-),3(-),4(-) &       & $\mathbf{84668.5}$  & 22652.6 & 1(+),2(+),4(+) & 83820.63 & 21008.33 & 1(-),2(+),3(-)\\
          &       &       & 0.1   &       & $\mathbf{87068}$      &    $\mathbf{87068}$     & $\mathbf{87068}$    &       &   85065.93 & 5953.13 & 2(-),3(-),4(-) & 86128.03 & 5931.25 & 1(+),3(+),4(-) &       & 85719.73 & 23177.6 & 1(+),2(-),4(-) & 86849.87 & 21986.46 & 1(+),2(+),3(+)  \\
          \hline
        uncorr 2  & 93559 & 25    & 0.001 &       &    143444   &   143448    &    143455   &     &  147143.5 & 6072.40 & 2(+),3(-),4(+) & 142829.3 & 6061.98 & 1(-),3(-),4(-) &       & $\mathbf{147754.5}$  & 24839.58 & 1(+),2(+),4(+) & 143393.7 & 24064.58 & 1(-),2(+),3(-)  \\
          &       &       & 0.01  &       &    101606   &    147780   &   147780    &         &     147539.5 & 6072.40 & 2(+),3(-),4(-) & 147115.2 & 6061.98 & 1(-),3(-),4(-) &       & $\mathbf{148125.7}$ & 25049.48 & 1(+),2(+),4(+) & 147605.7 & 24511.98 & 1(+),2(+),3(-)\\
          &       &       & 0.1   &       &   $\mathbf{149217}$     &     $\mathbf{149217}$  &    $\mathbf{149217}$   &       &   147936.9 & 6068.75 & 2(-),3(-),4(-) & 148518 & 6079.69 & 1(+),3(-),4(-) &       & 148594.9 & 24431.25 & 1(+),2(+),4(-) & 148960.7 & 25139.06 & 1(+),2(+),3(+) \\
          &       & 50    & 0.001 &       &      100660 &   120860    &    137219   &     &     145291.4 & 6063.54 & 2(+),3(-),4(+) & 136767 & 6048.44 & 1(-),3(-),4(-) &       & $\mathbf{145594.4}$   & 26981.25 & 1(+),2(+),4(+) & 137400 & 22478.13 & 1(-),2(+),3(-) \\
          &       &       & 0.01  &       &    101554   &   145726    &    145726   &   &    146067.2 & 6066.67 & 2(+),3(-),4(+) & 145189.8 & 6048.44 & 1(-),3(-),4(-) &       &$\mathbf{146326}$   & 27126.04 & 1(+),2(+),4(+) & 145655.9 & 24484.9 & 1(-),2(+),3(-) \\
          &       &       & 0.1   &       &   101648    &   $\mathbf{148531}$    &     $\mathbf{148531}$   &        &     146945.4 & 6070.83 & 2(-),3(-),4(-) & 147955.5 & 6075.52 & 1(+),3(+),4(-) &       & 147254 & 26751.56 & 1(+),2(-),4(-) & 148368.5 & 24806.77 & 1(+),2(+),3(+)  \\
          \hline
       uncorr 3   & 171819 & 25    & 0.001 &       &     138416  &   200244    &   204498    &        &     207630.3 & 6189.06 & 2(+),3(-),4(+) & 204214.7 & 6166.67 & 1(-),3(-),4(-) &       & $\mathbf{208249.2}$ & 22107.29 & 1(+),2(+),4(+) & 204790.2 & 20989.58 & 1(-),2(+),3(-) \\
          &       &       & 0.01  &       &   208260    &  208260     &    208260   &       &      207978.8 & 6205.73 & 2(+),3(-),4(-) & 207537.9 & 6172.92 & 1(-),3(-),4(-) &       &$\mathbf{208532}$ & 22053.13 & 1(+),2(+),4(+) & 208161.3 & 21864.06 & 1(+),2(+),3(-) \\
          &       &       & 0.1   &       & $\mathbf{ 209344}$     &    $\mathbf{ 209344}$   &  $\mathbf{ 209344}$     &       &      208266.9 & 6196.35 & 2(-),3(-),4(-) & 208671.4 & 6169.79 & 1(+),3(-),4(-) &       & 208889.4 & 22261.46 & 1(+),2(+),4(-) & 209180.5 & 21320.31 & 1(+),2(+),3(+)\\
          &       & 50    & 0.001 &       &      114287 &     144843  & 144843      &    &   206180.1 & 6193.75 & 2(+),3(-),4(+) & 199212.3 & 6169.79 & 1(-),3(-),4(-) &       & $\mathbf{ 206374.7}$ & 25970.83 & 1(+),2(+),4(+) & 199810.8 & 19647.4 & 1(-),2(+),3(-)  \\
          &       &       & 0.01  &       &    154981   &   156922    &  172798     &    &   206719.9 & 6193.23 & 2(+),3(-),4(+) & 206014 & 6175.00 & 1(-),3(-),4(-) &       &$\mathbf{ 206915.9}$ & 26699.48 & 1(+),2(+),4(+) & 206622.3 & 22066.67 & 1(-),2(+),3(-) \\
          &       &       & 0.1   &       &    208391   &    $\mathbf{ 208899}$   &  $\mathbf{208899 }$  &       &      207471.5 & 6190.10 & 2(-),3(-),4(-) & 208135.4 & 6184.38 & 1(+),3(+),4(-) &       & 207649 & 26861.46 & 1(+),2(-),4(-) & 208754& 21738.54 & 1(+),2(+),3(+)\\
          \bottomrule
    \end{tabular}}}%
  \label{tab:500deltaUniform}%
\end{table}%

\begin{figure}[t]
    \centering
\begin{tikzpicture}
    \begin{axis}[
        width  = 0.9* \textwidth,
        height = 8cm,
        major x tick style = transparent,
        ybar= 2*\pgflinewidth,
        bar width=8pt,
        ymajorgrids = true,
        ylabel = {Profit},
        symbolic x coords={$\mathbb{A}$,$\mathbb{B}$,$\mathbb{C}$,$\mathbb{D}$, $\mathbb{E}$,$\mathbb{F}$,$\mathbb{G}$},
        xtick = data,
        scaled y ticks = true,
        enlarge x limits= 0.1,
        ymin=0,
        legend cell align=left,
        legend style={
               at={(1,0)},
                anchor=south east,
                column sep=1ex
        }
    ]
        \addplot[style={bblue,fill=bblue,mark=none}]
            coordinates {
 ($\mathbb{A}$, 13701) ($\mathbb{B}$, 13707) ($\mathbb{C}$,15208 ) ($\mathbb{D}$,13614.8 ) ($\mathbb{E}$, 15104.57) ($\mathbb{F}$,13707 ) ($\mathbb{G}$,15208 )};

        \addplot[style={rred,fill=rred,mark=none}]
             coordinates { ($\mathbb{A}$, 15252) ($\mathbb{B}$, 15252) ($\mathbb{C}$,15348 ) ($\mathbb{D}$,15150.47) ($\mathbb{E}$, 15232.07) ($\mathbb{F}$,15252 ) ($\mathbb{G}$,15348 )};

        \addplot[style={ggreen,fill=ggreen,mark=none}]
             coordinates {($\mathbb{A}$, 15775) ($\mathbb{B}$, 15782) ($\mathbb{C}$,15599 ) ($\mathbb{D}$,15680.87) ($\mathbb{E}$, 15454.03) ($\mathbb{F}$,15782) ($\mathbb{G}$,15599 )};

        \legend{$\alpha=0.001$,$\alpha=0.01$,$\alpha=0.1$}
    \end{axis}
\end{tikzpicture}
    \caption{Comparison for different values of $\alpha$ with fixed $\delta=25$.}
    \label{fig:fixdelta}
\end{figure}

\begin{figure}
    \centering
\begin{tikzpicture}
    \begin{axis}[
        width  = 0.9* \textwidth,
        height = 8cm,
        major x tick style = transparent,
        ybar= 2*\pgflinewidth,
        bar width=8pt,
        ymajorgrids = true,
        ylabel = {Profit},
        symbolic x coords={$\mathbb{A}$,$\mathbb{B}$,$\mathbb{C}$,$\mathbb{D}$, $\mathbb{E}$,$\mathbb{F}$,$\mathbb{G}$},
        xtick = data,
        scaled y ticks = true,
        enlarge x limits= 0.1,
        ymin=0,
        legend cell align=left,
        legend style={
               at={(1,0)},
                anchor=south east,
                column sep=1ex
        }
    ]
        \addplot[style={bblue,fill=bblue,mark=none}]
            coordinates {
 ($\mathbb{A}$, 15252) ($\mathbb{B}$, 15252) ($\mathbb{C}$,15348 ) ($\mathbb{D}$,15150.47 ) ($\mathbb{E}$, 15232.07) ($\mathbb{F}$,15252 ) ($\mathbb{G}$,15348 )};

        \addplot[style={rred,fill=rred,mark=none}]
             coordinates { ($\mathbb{A}$, 14505) ($\mathbb{B}$, 14505) ($\mathbb{C}$,14742 ) ($\mathbb{D}$,14416.8) ($\mathbb{E}$, 14613.8) ($\mathbb{F}$,14505 ) ($\mathbb{G}$,14742 )};

        \legend{$\delta=25$,$\delta=50$}
    \end{axis}
\end{tikzpicture}
    \caption{Comparison for different values of $\delta$ with fixed $\alpha=0.01$.}
    \label{fig:fixalpha}
\end{figure}

The first insight from Table $\ref{tab:100deltaCheby}$ and Table \ref{tab:100deltaChernoff} is according to the deterministic approaches, the objective values obtained by the DP are at least as good as the results obtained by heuristic approach and ILP. When using Chebyshev's inequality to estimate the probability of violating the capacity, the heuristic approach performs at least as well as ILP. In observing the values listed in the \textit{Heuristic approach} and \textit{GSEMO} columns in Table \ref{tab:100deltaCheby}, we find that, in most of the instances, the results obtained using the heuristic approach are better than those obtained using GSEMO. Furthermore, the computation time associated with the heuristic approach is shorter than that associated with GSEMO. However, when applying the Chernoff bound to the problems, the results obtained using GSEMO are better than those obtained using the heuristic approach in most instances. Moreover, it can be observed from the tables that the GSEMO outperforms (1+1)~EA in combination with different estimate methods.

Furthermore, it can be seen from the values in the \textit{profit} columns that the profit decreases as the value of $\alpha$ decreases with a fixed value of $\delta$, and that the profit decreases when the value of $\delta$ increases, while the value of $\alpha$ stays constant. We take the first instances from Table \ref{tab:100deltaCheby} and Table \ref{tab:100deltaChernoff} as an examples, presenting the effects of $\delta$ and $\alpha$ in Figures \ref{fig:fixdelta} and \ref{fig:fixalpha}. Figure \ref{fig:fixdelta} shows how the chance-constrained bound (denoted as $\alpha$) affects the quality of the solution when the uncertainty of weights is fixed. In both figures, the $x$ labels are the algorithms combined with probability tools, where $\mathbb{A}$ denotes ILP(2mins)-Chebyshev's inequality,  $\mathbb{B}$ denotes heuristic approach-Chebyshev's inequality,  $\mathbb{C}$ denotes heuristic approach-Chernoff bound, $\mathbb{D}$ denotes (1+1)~EA-Chebyshev's inequality, $\mathbb{E}$ denotes (1+1)~EA-Chernoff bound, $\mathbb{F}$ denotes GSEMO-Chebyshev's inequality and $\mathbb{G}$ denotes GSEMO-Chernoff bound.

Figure \ref{fig:fixdelta} shows how the chance-constrained bound denoted as $\alpha$ affects the solution quality when the uncertainty of weights is fixed. Correspondingly, Figure \ref{fig:fixalpha} shows how the uncertainty of weights affects the solution quality when $\alpha$ is fixed. Both bar charts show the average values of (1+1)~EA and GSEMO with respect to probability inequality. In Figure \ref{fig:fixdelta}, the three bars of each group correspond to the value of $\alpha$ equal to $\{0.001,0.01,0.1\}$. In Figure \ref{fig:fixalpha}, the two bars of each group correspond to the value of $\delta = \{25, 50\}$. We can see from the figures that profits increase as the chance-constrained bound $\alpha$ increases or the uncertainty of weights $\delta$ decreases. Intuitively, this makes sense, as the relaxing of $\alpha$ or the tightening of $\delta$ allows the algorithms to compute solutions that are closer to the capacity and thereby increase their profit.

Next, we investigate the difference between the estimated methods. It can be observed from Table \ref{tab:500deltaUniform}, for all approaches, that when the value of $\alpha$ is small (i.e.,$\alpha=0.001$ or $0.01$), the profits obtained using the Chernoff bound are significantly better than those obtained using Chebyshev's inequality. In contrast, if $\alpha=0.1$, the results obtained using Chebyshev's inequality are significantly better than those obtained using the Chernoff bound. The experimental results match the theoretical implications of Theorem \ref{theo:compare} and our related discussion.

In Table \ref{tab:500deltaUniform}, the \textit{profit} under (1+1)~EA and GSEMO denotes the average profit of the $30$ runs, and \textit{stat} denotes the statistical comparison of the algorithms combined with the estimation methods. The numbers in the \textit{stat} column denote the significance of the results for an algorithm and constraint violation estimation method. For example, in Table \ref{tab:500deltaUniform}, the numbers $(2(+),3(-),4(+))$ listed in the \textit{stat} column in the first row under \textit{(1+1)~EA - Chernoff bound (1)} means that the current algorithm is significantly better than \textit{Chebyshev's inequality (2)} and \textit{Chebyshev's inequality (4)}, and significantly worse than \textit{Chernoff bound (3)}. 

We now compare the performances of (1+1)~EA with GSEMO. Looking at the values listed in the $\textit{stat}$ columns of Table \ref{tab:500deltaUniform}, it can be seen that the results listed under \textit{(1+1)~EA-Chernoff (1)} are significantly worse than those under \textit{GSEMO-Chernoff (3)}. A similar relationship exists between the values under \textit{(1+1)~EA-Chebyshev (2)} and \textit{GSEMO-Chebysehv(4)}, indicating that the performance of GSEMO is significantly better than that of (1+1)~EA under all probability tails.

\begin{table}[htbp]
  \centering
  \caption{Statistical results for the instance eil101 with 100 items based on multiplicative uniform distribution and Chebyshev's inequality}
  \scalebox{0.5}{
 \makebox[\linewidth][c]{
 \tabcolsep=0.1cm
    \begin{tabular}{crrrrrrrrrrrrrrrrrrr}
    \toprule
          & \multicolumn{1}{l}{Capacity} & \multicolumn{1}{l}{$\beta$} & \multicolumn{1}{l}{$\alpha$} &       & \multicolumn{3}{c}{ILP}              &       & \multicolumn{2}{c}{Heuristic approach}        &       & \multicolumn{2}{c}{DP}       &       & \multicolumn{2}{c}{(1+1)EA}       &       & \multicolumn{2}{c}{GSEMO}   \\
          &       &       &       &       & \multicolumn{1}{l}{2mins} & \multicolumn{1}{l}{6mins} & \multicolumn{1}{l}{10mins} &       & \multicolumn{1}{l}{profit} & \multicolumn{1}{l}{time(ms)} &       & \multicolumn{1}{l}{profit} & \multicolumn{1}{l}{time(ms)} &       & \multicolumn{1}{l}{profit} & \multicolumn{1}{l}{time(ms)} &       & \multicolumn{1}{l}{profits} & \multicolumn{1}{l}{time(ms)} \\
          \midrule
          bou-s-c 1 & 11775 & 0.01  & 0.001 &       & 15309 & 15309 & 15318 &       &$\mathbf{15350}$  & 690   &       & $\mathbf{15350}$  & 4810778 &       & 15230.67 & 1221.88 &       & 15349.6 & 62382.7675 \\
          &       &       & 0.01  &       & $\mathbf{15801}$ & $\mathbf{15801}$  & $\mathbf{15801}$  &       & $\mathbf{15801}$  & 1072  &       & $\mathbf{15801}$  & 12734208 &       & 15684.93 & 1234.9 &       & $\mathbf{15801}$  & 66815.46 \\
          &       &       & 0.1   &       & $\mathbf{15946}$  & $\mathbf{15946}$  & $\mathbf{15946}$  &       & $\mathbf{15946}$  & 1122  &       & $\mathbf{15946}$  & 26128269 &       & 15825.47 & 1218.75 &       & 15940.4 & 71332.5442 \\
          &       & 0.05  & 0.001 &       & 13100 & 13199 & 13199 &       & $\mathbf{13245}$  & 1369  &       & $\mathbf{13245}$  & 6293917 &       & 13123.27 & 1222.92 &       & $\mathbf{13245}$  & 55164.344 \\
          &       &       & 0.01  &       & 14905 & 14905 & 14905 &       & $\mathbf{15027}$  & 1073  &       & $\mathbf{15027}$ & 16956887 &       & 14926.2 & 1226.56 &       & 14992.4333 & 56806.4804 \\
          &       &       & 0.1   &       & 15702 & 15702 & 15702 &       & $\mathbf{15710}$ & 1181  &       & $\mathbf{15710}$  & 31257870 &       & 15581.1 & 1221.88 &       & $\mathbf{15710}$  & 59787.857 \\
          &       & 0.1   & 0.001 &       & 11128 & 11128 & 11128 &       & 11247 & 646   &       &$\mathbf{11270}$   & 6724128 &       & 11193.8 & 1222.4 &       & 11269.0667 & 53176.8783 \\
          &       &       & 0.01  &       & 13983 & 14002 & 14043 &       & $\mathbf{14143}$  & 1174  &       & $\mathbf{14143}$ & 19649101 &       & 14037.23 & 1222.4 &       & 14141.4667 & 59414.7318 \\
          &       &       & 0.1   &       & 15312 & 15346 & 15346 &       & $\mathbf{15384}$ & 1267  &       & $\mathbf{15384}$ & 35553950 &       & 15257.7 & 1220.31 &       & $\mathbf{15384}$ & 65737.6751 \\
          \hline
        bou-s-c 2  & 31027 & 0.01  & 0.001 &       & 36524 & 36634 & 36634 &       & $\mathbf{36728}$ & 5427  &     &   $-$      &$-$         &       & 36554.43 & 1264.58 &       & 36694.4667 & 116506.762 \\
          &       &       & 0.01  &       & 37475 & 37515 & 37515 &       & $\mathbf{37557}$ & 8078  &       &  $-$       &     $-$    &       & 37375.73 & 1261.46 &       & 37488.7 & 119125.093 \\
          &       &       & 0.1   &       & 37796 & 37796 & 37796 &       & $\mathbf{37819}$ & 14631 &       &   $-$     &     $-$    &       & 37679.77 & 1257.29 &       & 37816.6667 & 154614.705 \\
          &       & 0.05  & 0.001 &       &32069 & 32170 & 32239 &       & 32685 & 4349  &       &     $-$    &      $-$   &       & 32607.13 & 1277.08 &       &$\mathbf{32701.1}$   & 210753.913 \\
          &       &       & 0.01  &       & 35855 & 36002 & 36002 &       & $\mathbf{36069}$ & 5834  &       &  $-$       &   $-$      &       & 35903.27 & 1267.19 &       & 36044.3 & 200180.797 \\
          &       &       & 0.1   &       & 37284 & 37284 & 37284 &       & $\mathbf{37367 }$& 6016  &       &    $-$     &   $-$      &       & 37170.57 & 1259.38 &       & 37301.1667 & 172459.537 \\
          &       & 0.1   & 0.001 &       & 28246 & 28354 & 28354 &       & 28787 & 2575  &       &   $-$      &     $-$    &       & 28724.27 & 1286.46 &       & $\mathbf{28821 }$ & 91837.3079 \\
          &       &       & 0.01  &       & 34135 & 34135 & 34135 &       & $\mathbf{34419 }$ & 5018  &       &   $-$      &    $-$     &       & 34278.83 & 1280.21 &       & 34418.5333 & 99988.7866 \\
          &       &       & 0.1   &       & 36576 & 36701 & 36701 &       & $\mathbf{36784 }$ & 5886  &       &      $-$   &    $-$     &       & 36619.67 & 1263.02 &       & 36758.13 & 78642.88 \\
          \hline
       bou-s-c 3   & 58455 & 0.01  & 0.001 &       & 65667 & 65667 & 65667 &       & $\mathbf{65737 }$  & 63940 &       &   $-$      &    $-$     &       & 65510.7 & 1316.15 &       & 65690.53 & 97893.23 \\
          &       &       & 0.01  &       & 66958 & 67011 & 67011 &       & $\mathbf{67077}$ & 80608 &       &   $-$      &     $-$    &       & 66831.33 & 1311.46 &       & 66995.6 & 101069.27 \\
          &       &       & 0.1   &       & 67461 & 67461 & 67461 &       & $\mathbf{67480}$ & 151215 &       &      $-$   &   $-$      &       & 67290.13 & 1313.02 &       & 67439.1 & 100545.31 \\
          &       & 0.05  & 0.001 &       & 58770 & 58770 & 58770 &       & $\mathbf{59319}$ & 18667 &       & $-$        &    $-$     &       & 59223.87 & 1344.79 &       & 59304.3 & 84612.5 \\
          &       &       & 0.01  &       & 64472 & 64500 & 64500 &       &$\mathbf{64713}$  & 65809 &       &    $-$     &   $-$      &       & 64492.3 & 1320.83 &       & 64652.57 & 96270.83 \\
          &       &       & 0.1   &       & 66671 & 66671 & 66671 &       & $\mathbf{66742}$ & 69379 &       &  $-$       &  $-$       &       & 66507.07 & 1310.94 &       & 66691.73 & 100061.46 \\
          &       & 0.1   & 0.001 &       & 48564 & 51059 & 50689 &       & 53359 & 8798  &       &    $-$     &   $-$      &       & 53265.03 & 1358.33 &       &$\mathbf{ 53359.13}$ & 74005.21 \\
          &       &       & 0.01  &       & 61574 & 61657 & 61657 &       & 61974 & 58018 &       &   $-$      &   $-$      &       & 61924.97 & 1325  &       &$\mathbf{ 62063.7 }$& 90768.75 \\
          &       &       & 0.1   &       & 65674 & 65738 & 65738 &       &$\mathbf{ 65837 }$ & 47670 &       &      $-$   &      $-$   &       & 65633.57 & 1319.27 &       & 65774.27 & 97894.79 \\
          \hline
    uncorr 1  & 7715  & 0.01  & 0.001 &       & \cellcolor[gray]{.9} $\mathbf{17064}$ & \cellcolor[gray]{.9}{4s} &       &       & $\mathbf{17064}$ & 204   &       & $\mathbf{17064}$ & 31106720 &       & 16989.4 & 1218.23 &       & $\mathbf{17064}$ & 8394.56443 \\
          &       &       & 0.01  &       & \cellcolor[gray]{.9}$\mathbf{17366}$ & \cellcolor[gray]{.9}{3.08s} &       &       & $\mathbf{17366}$ & 191   &       & $\mathbf{17366}$ & 49253218 &       & 17346.17 & 1229.69 &       & $\mathbf{17366}$ & 8361.6105 \\
          &       &       & 0.1   &       & \cellcolor[gray]{.9} $\mathbf{17499}$ & \cellcolor[gray]{.9} {1.39s} &       &       & $\mathbf{17499}$  & 199   &       & $\mathbf{17499}$  & 56877261 &       & 17444 & 1263.54 &       & $\mathbf{17499}$  & 6842.4115 \\
          &       & 0.05  & 0.001 &       & 15322 & 15322 & 15322 &       &$\mathbf{15385}$   & 151   &       & $\mathbf{15385}$   & 13416371 &       & 15293.33 & 1210.94 &       & $\mathbf{15385}$   & 7834.6693 \\
          &       &       & 0.01  &       & $\mathbf{16875}$   & $\mathbf{16875}$  & $\mathbf{16875}$  &       & $\mathbf{16875}$  & 185   &       & $\mathbf{16875}$  & 19806209 &       & 16808.23 & 1219.27 &       & $\mathbf{16875}$  & 7505.27357 \\
          &       &       & 0.1   &       & \cellcolor[gray]{.9}$\mathbf{17325}$ & \cellcolor[gray]{.9}{4.64s} &       &       & $\mathbf{17325}$ & 176   &       & $\mathbf{17325}$ & 39985165 &       & 17290.97 & 1218.23 &       & $\mathbf{17325}$ & 6745.7595 \\
          &       & 0.1   & 0.001 &       &$\mathbf{13589}$  & $\mathbf{13589}$ & $\mathbf{13589}$ &       & $\mathbf{13589}$ & 70    &       & $\mathbf{13589}$ & 463375 &       & 13507.6 & 1209.38 &       & $\mathbf{13589}$ & 4679.62223 \\
          &       &       & 0.01  &       & \cellcolor[gray]{.9} $\mathbf{16143}$ & \cellcolor[gray]{.9}{60s} &       &       & $\mathbf{16143}$ & 117   &       & $\mathbf{16143}$ & 9685826 &       & 16047.6 & 1216.15 &       & $\mathbf{16143}$ & 6412.02927 \\
          &       &       & 0.1   &       & \cellcolor[gray]{.9}$\mathbf{17064}$ & \cellcolor[gray]{.9}{60s} &       &       & $\mathbf{17064}$  & 158   &       & $\mathbf{17064}$  & 28925543 &       & 16994.47 & 1218.75 &       & $\mathbf{17064}$  & 9153.54877 \\
          \hline
        uncorr 2  & 19545 & 0.01  & 0.001 &       & 29075 & 29085 & 29085 &       &$\mathbf{29106}$   & 940   &       & $-$    & $-$   &       & $\mathbf{29106}$  & 1265.63 &       & $\mathbf{29106}$  & 15434.3676 \\
          &       &       & 0.01  &       &\cellcolor[gray]{.9} $\mathbf{29466}$ & \cellcolor[gray]{.9}{16.52s} &       &       & $\mathbf{29466}$ & 1195  &       &  $-$       &   $-$      &       & 29390.6 & 1165.63 &       & $\mathbf{29466}$ & 15147.9511 \\
          &       &       & 0.1   &       & \cellcolor[gray]{.9} $\mathbf{29585}$  &  \cellcolor[gray]{.9}{3.01s} &       &       & $\mathbf{29585}$  & 1001  &       &  $-$       &   $-$      &       & 29521.43 & 1265.1 &       & $\mathbf{29585}$  & 21475.098 \\
          &       & 0.05  & 0.001 &       & 27087 & 27087 & 27087 &       &$\mathbf{27096}$   & 814   &       &    $-$     &   $-$      &       & 26974.73 & 1256.25 &       & $\mathbf{27096}$  & 18433.7486 \\
          &       &       & 0.01  &       & 28740 & 28815 & 28742 &       & $\mathbf{28815}$ & 728   &       &   $-$      &   $-$      &       & 28687.23 & 1267.71 &       & $\mathbf{28815}$ & 15621.9082 \\
          &       &       & 0.1   &       &\cellcolor[gray]{.9} $\mathbf{29415}$  & \cellcolor[gray]{.9}{118.91s} &       &       &  $\mathbf{29415}$  & 656   &       &  $-$       &  $-$       &       & 29327.2 & 1266.67 &       & $\mathbf{29415}$  & 22129.7554 \\
          &       & 0.1   & 0.001 &       & 25006 & 25033 & 25006 &       &$\mathbf{25043}$   & 540   &       &   $-$      &  $-$       &       & 24941.5 & 1248.44 &       & $\mathbf{25043}$  & 15159.6398 \\
          &       &       & 0.01  &       &$\mathbf{27985}$   & $\mathbf{27985}$ & $\mathbf{27985}$ &       & $\mathbf{27985}$ & 958   &       &   $-$      &   $-$      &       & 27899.77 & 1261.98 &       & $\mathbf{27985}$ & 14720.0369 \\
          &       &       & 0.1   &       &$\mathbf{29155}$  & $\mathbf{29155}$  & $\mathbf{29155}$  &       & $\mathbf{29155}$  & 740   &       &  $-$       &   $-$      &       & 29072.53 & 1265.63 &       & $\mathbf{29155}$  & 15869.881 \\
          \hline
      uncorr 3  & 36091 & 0.01  & 0.001 &       & 40651 & $\mathbf{40672}$  &$\mathbf{40672}$   &       & $\mathbf{40672}$  & 1351  &       &   $-$      &   $-$      &       & 40633.83 & 1342.19 &       & 40671.7 & 55625.4561 \\
          &       &       & 0.01  &       &\cellcolor[gray]{.9} $\mathbf{40994}$ & \cellcolor[gray]{.9}{57.04s} &       &       & $\mathbf{40994}$  & 1268  &       &  $-$       &   $-$      &       & 40938.17 & 1342.71 &       & $\mathbf{40994}$  & 53580.4838 \\
          &       &       & 0.1   &       &\cellcolor[gray]{.9} $\mathbf{41091}$ & \cellcolor[gray]{.9}{1.82s} &       &       & $\mathbf{41091}$ & 1264  &       &   $-$      &      $-$   &       & 41054.23 & 1342.19 &       & 41090.9333 & 62323.5451 \\
          &       & 0.05  & 0.001 &       & 38742 & 38742 & 38742 &       & $\mathbf{38743}$ & 1414  &       &      $-$   &   $-$      &       & 38712.4 & 1328.13 &       & $\mathbf{38743}$ & 53287.6329 \\
          &       &       & 0.01  &       &$\mathbf{40463}$  & $\mathbf{40463}$ & $\mathbf{40463}$ &       & $\mathbf{40463}$ & 1587  &       & $-$   &    $-$     &       & 40348.77 & 1339.58 &       & $\mathbf{40463}$ & 55654.3928 \\
          &       &       & 0.1   &       &\cellcolor[gray]{.9} $\mathbf{40985}$ & \cellcolor[gray]{.9}{82.48s} &       &       & $\mathbf{40985}$  & 2100  &       &  $-$   &   $-$      &       & 40869.57 & 1345.31 &       & $\mathbf{40985}$  & 51387.5036 \\
          &       & 0.1   & 0.001 &       & 35650 & 35978 & 36305 &       & $\mathbf{36390}$ & 572   &       &     $-$    &    $-$     &       & 36317.6 & 1314.06 &       &  $\mathbf{36390}$  & 42624.4348 \\
          &       &       & 0.01  &       &  $\mathbf{39628}$  &  $\mathbf{39628}$ &  $\mathbf{39628}$ &       &  $\mathbf{39628}$ & 771   &       &    $-$     &   $-$      &       & 39587.9 & 1333.85 &       &  $\mathbf{39628}$ & 48509.0626 \\
          &       &       & 0.1   &       & 40668 & 40668 & 40668 &       & $\mathbf{40672}$  & 1047  &       &  $-$       &  $-$       &       & 40649.03 & 1342.19 &       & 40671.8 & 50373.7936 \\
          \bottomrule
\end{tabular}}}%
\label{tab:100betaUniform}%
\end{table}

\begin{table}[htbp]
  \centering
  \caption{Statistical results for the instance eil101 with 500 items based on multiplicative uniform distribution and Chebyshev's inequality}
  \scalebox{0.5}{
 \makebox[\linewidth][c]{
 \tabcolsep=0.1cm
    \begin{tabular}{crrrrrrrrrrlrrrl}
    \toprule
          & \multicolumn{1}{l}{Capacity} & \multicolumn{1}{l}{$\beta$} & \multicolumn{1}{l}{$\alpha$} &       & \multicolumn{3}{c}{ILP} &       & \multicolumn{3}{c}{(1+1)EA(1)} &       & \multicolumn{3}{c}{GSEMO(2)} \\
          &       &       &       &       & \multicolumn{1}{l}{2mins} & \multicolumn{1}{l}{6mins} & \multicolumn{1}{l}{10mins} &       & \multicolumn{1}{l}{profit} & \multicolumn{1}{l}{time(ms)} & stat  &       & \multicolumn{1}{l}{profits} & \multicolumn{1}{l}{time(ms)} & stat \\
          \midrule
   bou-s-c 1 & 61447 & 0.01  & 0.001 &       & $\mathbf{ 77951}$ &  $\mathbf{ 77951}$ &  $\mathbf{ 77951}$ &       & 76680.71 & 5941.67 & 2(-)  &       & 77364.63 & 38885.42 & 1(+) \\
          &       &       & 0.01  &       & 78649 & 78666 &  $\mathbf{ 78749}$ &       & 77757.42 & 5928.65 & 2(-)  &       & 78423.77 & 39448.44 & 1(+) \\
          &       &       & 0.1   &       &$\mathbf{ 79411}$ & $\mathbf{ 79411}$  & $\mathbf{ 79411}$  &       & 77937.93 & 5945.83 & 2(-)  &       & 78774.67 & 39529.69 & 1(+) \\
          &       & 0.05  & 0.001 &       & 71001 & 71825 & 71825 &       & 71273.63 & 5932.81 & 2(-)  &       & $\mathbf{71866.73}$  & 38789.06 & 1(+) \\
          &       &       & 0.01  &       & $\mathbf{76766 }$&  $\mathbf{76766 }$ &  $\mathbf{76766 }$ &       & 75786.83 & 5937.5 & 2(-)  &       & 76556.33 & 39010.94 & 1(+) \\
          &       &       & 0.1   &       &  $\mathbf{78751 }$ & $\mathbf{78751 }$ & $\mathbf{78751 }$ &       & 77472.63 & 5928.13 & 2(-)  &       & 78160.3 & 39192.19 & 1(+) \\
          &       & 0.1   & 0.001 &       & 44200 & 49668 & 49672 &       & 65449.73 & 5949.48 & 2(-)  &       &$\mathbf{66057.46 }$  & 37460.94 & 1(+) \\
          &       &       & 0.01  &       & 45378 & $\mathbf{74709}$ & $\mathbf{74709}$ &       & 73659.27 & 5943.23 & 2(-)  &       & 74308.9 & 38843.75 & 1(+) \\
          &       &       & 0.1   &       & 77951 & 77951 & $\mathbf{78011}$ &       & 76666.37 & 5928.13 & 2(-)  &       & 77456.63 & 39206.25 & 1(+) \\
          \hline
       bou-s-c 2   & 162943 & 0.01  & 0.001 &       & 188693 &$\mathbf{189540}$  & $\mathbf{189540}$ &       & 187857.09 & 6138.54 & 2(-)  &       & 187986.63 & 62129.69 & 1(+) \\
          &       &       & 0.01  &       &$\mathbf{191716}$  & $\mathbf{191716}$  & $\mathbf{191716}$  &       & 190031.37 & 6133.33 & 2(-)  &       & 190110.66 & 62969.79 & 1(+) \\
          &       &       & 0.1   &       & 192263 & 192329 & $\mathbf{192437}$   &       & 190653.74 & 6133.33 & 2(-)  &       & 190824.75 & 63540.1 & 1(+) \\
          &       & 0.05  & 0.001 &       & 175165 & 175870 & 175870 &       & 176974.32 & 6151.04 & 2(-)  &       & $\mathbf{177249.9}$  & 61518.04 & 1(+) \\
          &       &       & 0.01  &       & 187042 & $\mathbf{187614}$ & $\mathbf{187614}$ &       & 186179.49 & 6133.85 & 2(-)  &       & 186342.34 & 62348.44 & 1(+) \\
          &       &       & 0.1   &       & 191045 &$\mathbf{191098}$  & $\mathbf{191098}$  &       & 189401.84 & 6130.73 & 2(-)  &       & 189591.47 & 62891.67 & 1(+) \\
          &       & 0.1   & 0.001 &       & 120607 & 147427 & 155490 &       & 165606.16 & 6167.19 & 2(-)  &       & $\mathbf{165991.14 }$ & 57807.81 & 1(+) \\
          &       &       & 0.01  &       & 181967 &  $\mathbf{182660 }$&  $\mathbf{182660 }$ &       & 181746.03 & 6143.23 & 2(-)  &       & 181972.58 & 61254.69 & 1(+) \\
          &       &       & 0.1   &       & $\mathbf{189477 }$ & $\mathbf{189477 }$  & $\mathbf{189477 }$  &       & 188041.25 & 6126.56 & 2(-)  &       & 188168.17 & 62590.63 & 1(+) \\
          \hline
       bou-s-c 3   & 307286 & 0.01  & 0.001 &       & 343297 &  $\mathbf{343820 }$  & $\mathbf{343820 }$   &       & 341568.23 & 6425  & 2(-)  &       & 342432.33 & 85480.73 & 1(+) \\
          &       &       & 0.01  &       & 346271 & 347015 & $\mathbf{347145 }$    &       & 344745.61 & 6421.83 & 2(-)  &       & 344843.61 & 85917.19 & 1(+) \\
          &       &       & 0.1   &       & 347820 &  $\mathbf{348151}$ & $\mathbf{348151}$  &       & 345852.22 & 6406.79 & 2(-)  &       & 345952.21 & 86294.27 & 1(+) \\
          &       & 0.05  & 0.001 &       & 168902 & 168902 & 203124 &       & 324236.08 & 6435.94 & 2(-)  &       & $\mathbf{324561.56}$   & 81405.73 & 1(+) \\
          &       &       & 0.01  &       & 199323 & 248792 & 311808 &       & 338751.13 & 6425.52 & 2(-)  &       &$\mathbf{338830.56}$  & 85056.25 & 1(+) \\
          &       &       & 0.1   &       &$\mathbf{346271}$  & $\mathbf{346271}$  & $\mathbf{346271}$  &       & 343994.57 & 6425  & 2(-)  &       & 344849.08 & 86061.98 & 1(+) \\
          &       & 0.1   & 0.001 &       & 167882 & 231201 & 231201 &       & 305712.67 & 6461.98 & 2(-)  &       &$\mathbf{306298.87}$   & 76398.96 & 1(+) \\
          &       &       & 0.01  &       & 169122 & 254642 & 254642 &       & 331962.13 & 6431.77 & 2(-)  &       & $\mathbf{331991.33}$  & 82465.1 & 1(+) \\
          &       &       & 0.1   &       & 200994 & 341453 &$\mathbf{343392}$   &       & 341835.52 & 6425  & 2(-)  &       & 341914.53 & 84382.29 & 1(+) \\
          \hline
           uncorr 1 & 37686 & 0.01  & 0.001 &       &    87252   &     87333  &  $\mathbf{87465 }$    &       & 86527.83 & 5964.58 & 2(-)  &       & 87259.37 & 22572.92 & 1(+) \\
          &       &       & 0.01  &       &  88093     &  $\mathbf{88104 }$        &   $\mathbf{88104 }$       &       & 87167.43 & 5961.46 & 2(-)  &       & 87897.4 & 22578.13 & 1(+) \\
          &       &       & 0.1   &       &   $\mathbf{88311 }$       &      $\mathbf{88311 }$    &    $\mathbf{88311 }$      &       & 87314.5 & 5963.54 & 2(-)  &       & 88099.57 & 22747.4 & 1(+) \\
          &       & 0.05  & 0.001 &       &    82872   &  82872  &    82872   &       & 82862.63 & 5951.56 & 2(-)  &       & $\mathbf{83631.43}$    & 22251.56 & 1(+) \\
          &       &       & 0.01  &       &  $\mathbf{  86847}$   &  $\mathbf{  86847}$     &   $\mathbf{  86847}$     &       & 85984.17 & 5953.65 & 2(-)  &       & 86720.27 & 22438.54 & 1(+) \\
          &       &       & 0.1   &       &  16687     &    87904   & $\mathbf{  87926}$       &       & 87003.1 & 5959.38 & 2(-)  &       & 87759.4 & 22661.98 & 1(+) \\
          &       & 0.1   & 0.001 &       &  16590     & 58953      &    75707   &       & 78590.73 & 5934.38 & 2(-)  &       &  $\mathbf{79491.4}$  & 21914.06 & 1(+) \\
          &       &       & 0.01  &       &   84948    &  $\mathbf{85449}$     &    $\mathbf{85449}$      &       & 84462.23 & 5947.92 & 2(-)  &       & 85287.3 & 22295.83 & 1(+) \\
          &       &       & 0.1   &       &   16687    &    87253   &    $\mathbf{87489}$      &       & 86561.52 & 5955.21 & 2(-)  &       & 87314.33 & 22748.96 & 1(+) \\
          \hline
        uncorr 2  & 93559 & 0.01  & 0.001 &       &   101601    &  $\mathbf{148534}$       &   $\mathbf{148534}$      &       & 147858.42 & 6346.31 & 2(-)  &       & 147978.3 & 30017.71 & 1(+) \\
          &       &       & 0.01  &       &    37741   &   149398    &  $\mathbf{149403}$       &       & 148717.23 & 6253.65 & 2(-)  &       & 148800.6 & 30072.4 & 1(+) \\
          &       &       & 0.1   &       &    37741   &   149695    & $\mathbf{149696}$    &       & 149021.6 & 6248.44 & 2(-)  &       & 149097 & 29593.75 & 1(+) \\
          &       & 0.05  & 0.001 &       &   101527    &  $\mathbf{143671}$        &   $\mathbf{143671}$        &       & 143106.43 & 6220.31 & 2(-)  &       &  $\mathbf{143671}$   & 98357.81 & 1(+) \\
          &       &       & 0.01  &       &   101593    &   142271    &   $\mathbf{147820}$      &       & 147109.2 & 6248.44 & 2(-)  &       & 147676.7 & 99925.52 & 1(+) \\
          &       &       & 0.1   &       &     148834  &  148834  &  $\mathbf{149218}$        &       & 148500.83 & 6247.4 & 2(-)  &       & 149050.2 & 101393.23 & 1(+) \\
          &       & 0.1   & 0.001 &       &  37939     &    114416   &   114416    &       & 137665.33 & 6257.29 & 2(-)  &       & $\mathbf{137942.4}$  & 28286.98 & 1(+) \\
          &       &       & 0.01  &       &   101587    &    143082   &   $\mathbf{145893}$      &       & 145231.46 & 6267.71 & 2(-)  &       & 145816 & 99587.5 & 1(+) \\
          &       &       & 0.1   &       &  148605     &   $\mathbf{148613}$    &    $\mathbf{148613}$     &       & 147931.4 & 6655.73 & 2(-)  &       & 148450 & 100575.52 & 1(+) \\
          \hline
         uncorr 3 & 171819 & 0.01  & 0.001 &       &  69026     &   $\mathbf{208530}$      &     $\mathbf{208530}$    &       & 207914.33 & 6685.42 & 2(-)  &       & 207955.3 & 114927.6 & 1(+) \\
          &       &       & 0.01  &       &    69076   &  $\mathbf{209489}$       &    $\mathbf{209489}$    &       & 208724.56 & 6684.8 & 2(-)  &       & 208737 & 114892.71 & 1(+) \\
          &       &       & 0.1   &       &  69076     &   $\mathbf{209734}$     &    $\mathbf{209734}$    &       & 208963.6 & 6617.71 & 2(-)  &       & 208980.9 & 114864.58 & 1(+) \\
          &       & 0.05  & 0.001 &       &   68118    &    141171   &    156856   &       & 203217.33 & 6648.96 & 2(-)  &       & $\mathbf{203416}$   & 113179.69 & 1(+) \\
          &       &       & 0.01  &       &   129569    &    207365   &    $\mathbf{207985}$   &       & 207275.63 & 6617.71 & 2(-)  &       & 207315.8 & 114427.6 & 1(+) \\
          &       &       & 0.1   &       &    129871   &  209229     &     $\mathbf{209255}$  &       & 208519.46 & 6657.29 & 2(-)  &       & 208567 & 114638.02 & 1(+) \\
          &       & 0.1   & 0.001 &       &   66564    &  142764     &  154311     &       & 197404.69 & 6572.4 & 2(-)  &       &$\mathbf{197631.8 }$  & 111736.46 & 1(+) \\
          &       &       & 0.01  &       &   68197    &  204818     &     $\mathbf{206030}$   &       & 205424.63 & 6635.42 & 2(-)  &       & 205511.3 & 113712.5 & 1(+) \\
          &       &       & 0.1   &       &  115721     &   208706    &      $\mathbf{208706}$  &       & 208053.79 & 6643.23 & 2(-)  &       & 208019.7 & 114743.23 & 1(+) \\
          \bottomrule
    \end{tabular}}}%
  \label{tab:500betaUniform}%
\end{table}%

In the second category of instances, the weights of items have multiplicative uniform distributions and take values in the real interval $[(1-\beta)a_i,(1+\beta)a_i]$. Here, the uncertainty gap corresponding to each item is expressed as a percentage of the expected weight. In this setting, the distance of the random intervals between each item vary for each item pair. We apply Chebyshev's inequality to deal with the chance constraint in this category of instances. The results are listed in Table \ref{tab:100betaUniform} and Table \ref{tab:500betaUniform}, and we mark the best result across all approaches in each row in bold. The grey cubes in the table highlight the instances in which ILP can be completed in the computation time listed in the second columns under the title \textit{ILP}. It can be observed that all algorithms produce inferior solutions when the uncertainty added to the weights of items (measured by $\beta$) increases, or the upper bound of the probability of overloading the capacity in the form of $\alpha$ decreases. 

Furthermore, it can be observed from Table \ref{tab:100betaUniform}, we found that the results obtained using the heuristic approach are better than those obtained using EAs, and that results obtained through the heuristic approach are more significant in the case of \textit{uncorrelated} instances. However, the results for instances which have 500 items cannot be list in Table \ref{tab:500betaUniform}, since it takes more than two hours for the heuristic approach to obtain a result in such cases. Moreover, the results listed in the \textit{ILP} columns are better than those in the \textit{(1+1)~EA} columns, but it is difficult to make a clear comparison of these results with the results of \textit{GSEMO}. The computation time associated with ILP is longer than that associated with EAs for most of the instances. Besides, as shown in Table \ref{tab:500betaUniform} the performance of GSEMO is significantly better than that of (1+1)~EA for all instances. 

\begin{table}[t]
  \centering
  \caption{Statistic results for the instance eil101 with 100 items based on Normal distribution and Chebyshev's inequality}
  \scalebox{0.5}{
 \makebox[\linewidth][c]{
 \tabcolsep=0.1cm
    \begin{tabular}{crrrrrrrrrrrrrrrrrrr}
    \toprule
          & \multicolumn{1}{l}{Capacity} & \multicolumn{1}{l}{$\beta$} & \multicolumn{1}{l}{$\alpha$} &       & \multicolumn{3}{c}{ILP} &       & \multicolumn{2}{c}{Heuristic approach} &       & \multicolumn{2}{c}{DP} &       & \multicolumn{2}{c}{(1+1)EA} &       & \multicolumn{2}{c}{GSEMO} \\
          &       &       &       &       & \multicolumn{1}{l}{2mins} & 6mins & \multicolumn{1}{l}{10mins} &       & \multicolumn{1}{l}{profit} & \multicolumn{1}{l}{time(ms)} &       & \multicolumn{1}{l}{profit} & \multicolumn{1}{l}{time(ms)} &       & \multicolumn{1}{l}{profit} & \multicolumn{1}{l}{time(ms)} &       & \multicolumn{1}{l}{profits} & \multicolumn{1}{l}{time(ms)} \\
          \midrule
          bou-s-c 1 & 11775 & 0.01  & 0.001 &       & $\mathbf{15635}$ & $\mathbf{15635}$ & $\mathbf{15635}$ &       & $\mathbf{15635}$ & 1074  &       &    $-$     &     $-$    &       & 15481.8 & 1227.083 &       & $\mathbf{15635}$ & 122546.3 \\
          &       &       & 0.01  &       &\cellcolor[gray]{.9}$\mathbf{15946}$  &\cellcolor[gray]{.9} 59.92s &       &       & $\mathbf{15946}$  & 1836  &       &     $-$    & $-$        &       & 15759.4 & 1232.813 &       & $\mathbf{15946}$  & 119314.8 \\
          &       &       & 0.1   &       &\cellcolor[gray]{.9} $\mathbf{15946}$ &\cellcolor[gray]{.9} 59.51s &       &       & $\mathbf{15946}$ & 1211  &       &   $-$      &   $-$      &       & 15862.53 & 1223.958 &       & $\mathbf{15946}$ & 114167.2 \\
          &       & 0.05  & 0.001 &       & 15114 & 15114 & 15114 &       &$\mathbf{15128}$  & 1316  &       &    $-$     &  $-$       &       & 14971.97 & 1224.479 &       & $\mathbf{15128}$ & 95547.55 \\
          &       &       & 0.01  &       & 15732 & 15732 & 15732 &       & $\mathbf{15736}$   & 1360  &       &     $-$    &  $-$       &       & 15631.53 & 1222.917 &       & $\mathbf{15736}$ & 98784.17 \\
          &       &       & 0.1   &       & \cellcolor[gray]{.9} $\mathbf{15946}$  & \cellcolor[gray]{.9}16.88s &       &       & $\mathbf{15946} $ & 1546  &       &   $-$      &  $-$       &       & 15789.63 & 1225  &       & $ \mathbf{15946} $  & 113056.9 \\
          &       & 0.1   & 0.001 &       & 14719 & 14719 &   14719    &       &$\mathbf{14739}$  & 971   &       &   $-$      &     $-$    &       & 14631.63 & 1224.479 &       & $\mathbf{14739}$  & 100566.2 \\
          &       &       & 0.01  &       & 15630 & 15630 &  15630     &       &$\mathbf{15635}$  & 2152  &       &    $-$     & $-$        &       & 15480.13 & 1226.042 &       & $\mathbf{15635}$  & 110594.3 \\
          &       &       & 0.1   &       & \cellcolor[gray]{.9}$\mathbf{15946}$  & \cellcolor[gray]{.9}13.94s &       &       & $\mathbf{15946}$  & 1634  &       &  $-$       &   $-$      &       & 15758.67 & 1221.354 &       & $\mathbf{15946}$  & 112109.6 \\
          \hline
       bou-s-c 2   & 31027 & 0.01  & 0.001 &       & $\mathbf{37359}$ &     $\mathbf{37359}$   &   $\mathbf{37359}$     &       &  $\mathbf{37359}$ & 22126 &       &   $-$      &   $-$      &       & 37159.3 & 1271.35 &       & 37374.47 & 326773.7 \\
          &       &       & 0.01  &       & 37272 & 37272& 37751 &       & $\mathbf{37752}$ & 28381 &       &  $-$       &  $-$       &       & 37586.77 & 1265.625 &       & 37751.9 & 265563.1 \\
          &       &       & 0.1   &       &\cellcolor[gray]{.9} $\mathbf{37874 }$&   \cellcolor[gray]{.9}36.48s    &       &       & $\mathbf{37874 }$ & 27277 &       &   $-$      &  $-$       &       & 37703.33 & 1268.75 &       & $\mathbf{37874 }$ & 399498.2 \\
          &       & 0.05  & 0.001 &       & 36539 & 36539&  36539     &       & $\mathbf{36606  }$& 20569 &       &   $-$      &  $-$       &       & 36425.97 & 1269.271 &       & 36605.83 & 446766.4 \\
          &       &       & 0.01  &       & 37509 & 37515 & 37515 &       &$\mathbf{37537  }$  & 22290 &       &  $-$       &  $-$       &       & 37345.97 & 1274.479 &       & 37536.07 & 541687.4 \\
          &       &       & 0.1   &       & 37776 & 37791& 37791 &       & $\mathbf{37809  }$  & 27942 &       &   $-$      &   $-$      &       & 37622.23 & 1266.667 &       & 37808.97 & 523256 \\
          &       & 0.1   & 0.001 &       & 35982 & 36000 & 36000 &       & $\mathbf{36103  }$  & 35383 &       &  $-$       &   $-$      &       & 35857.93 & 1281.771 &       & 36100.5 & 336191.5 \\
          &       &       & 0.01  &       & 37272 & 37303 & 37303 &       & $\mathbf{37375}$& 32159 &       &      $-$   &   $-$      &       & 37163.57 & 1286.979 &       & 37374.5 & 292882 \\
          &       &       & 0.1   &       & 37752 & 37752 &    37752   &       &$\mathbf{37760}$   & 34877 &       &  $-$       & $-$        &       & 37587.67 & 1266.667 &       & $\mathbf{37760}$ & 166735.9 \\
          \hline
       bou-s-c 3   & 58455 & 0.01  & 0.001 &       & 66883 & 66883& 66889 &       & $\mathbf{66895}$  & 263962 &       &    $-$     &   $-$      &       & 66633 & 1310.93 &       & 66856.3 & 253407.8 \\
          &       &       & 0.01  &       & 67402 & 67402 & 67402 &       &$\mathbf{67414}$  & 307935 &       &   $-$     &   $-$      &       & 67214.13 & 1410.93 &       & 67412.17 & 252614.6 \\
          &       &       & 0.1   &       & 67581 &    67581   &   67581    &       &$\mathbf{67582}$   & 217621 &       &     $-$    &   $-$      &       & 67395.43 & 1314.583 &       & 67580.73 & 250042.7 \\
          &       & 0.05  & 0.001 &       & 65863 & 65865&     65865  &       &$\mathbf{65871}$   & 178900 &       &     $-$    &    $-$     &       & 65674.87 & 1324.479 &       & 65868.4 & 244173.4 \\
          &       &       & 0.01  &       & 67067 & 67085 & 67085 &       & $\mathbf{67119}$  & 226313 &       &   $-$      &  $-$       &       & 66903.47 & 1328.125 &       & 67104.73 & 250305.2 \\
          &       &       & 0.1   &       & 67492 & 67492& 67492 &       & $\mathbf{67493}$  & 226818 &       &   $-$      &   $-$      &       & 67284.2 & 1314.583 &       & 67491.37 & 250183.9 \\
          &       & 0.1   & 0.001 &       & 65181 &$\mathbf{65187}$  & $\mathbf{65187}$ &       & $\mathbf{65187}$ & 193779 &       &  $-$       &  $-$       &       & 64911.7 & 1317.188 &       & 65102.6 & 241117.2 \\
          &       &       & 0.01  &       & 66883 & 66883 & 66883 &       & $\mathbf{66899}$ & 230572 &       & $-$        &   $-$      &       & 66675.67 & 1311.458 &       & 66851.1 & 249634.4 \\
          &       &       & 0.1   &       & 67425 & 67425 & 67425 &       &$\mathbf{67426}$  & 215087 &       &     $-$    &  $-$       &       & 67230.03 & 1327.083 &       & 67423.43 & 246294.3 \\
          \hline
    uncorr 1 & 7715  & 0.01  & 0.001 &        & \cellcolor[gray]{.9} $\mathbf{17151}$  &\cellcolor[gray]{.9} 155.74s & &      & $\mathbf{17151}$ & 170   &       & $\mathbf{17151}$ & 62462460 &       & 17099.73 & 1230.208 &       & $\mathbf{17151}$ & 7984.781 \\
          &       &       & 0.01  &       &\cellcolor[gray]{.9}  $\mathbf{17366 }$ & \cellcolor[gray]{.9}5.22s &       &       & $\mathbf{17366 }$ & 153   &       & $\mathbf{17366 }$ & 74021365 &       & 17333.5 & 1243.75 &       & $\mathbf{17366 }$ & 7432.983 \\
          &       &       & 0.1   &       & \cellcolor[gray]{.9}$\mathbf{17499}$   & \cellcolor[gray]{.9} 1.92s &       &       & $\mathbf{17499}$ & 342   &       & $\mathbf{17499}$ & 81548280 &       & 17445.2 & 1322.396 &       & $\mathbf{17499}$ & 6449.72 \\
          &       & 0.05  & 0.001 &       & $\mathbf{16638}$  & $\mathbf{16638}$ &   $\mathbf{16638}$    &       & $\mathbf{16638}$ & 127   &       & $\mathbf{16638}$ & 46226151 &       & 16571.3 & 1227.604 &       & $\mathbf{16638}$ & 8861.319 \\
          &       &       & 0.01  &       & \cellcolor[gray]{.9} $\mathbf{17247}$ & \cellcolor[gray]{.9} 15.46s &       &       & $\mathbf{17247}$ & 109   &       & $\mathbf{17247}$ & 66880112 &       & 17167.67 & 1229.167 &       & $\mathbf{17247}$ & 6220.258 \\
          &       &       & 0.1   &       &  \cellcolor[gray]{.9} $\mathbf{17408}$ &  \cellcolor[gray]{.9} 3.77s &       &       & $\mathbf{17408}$ & 164   &       & $\mathbf{17408}$ & 71001731 &       & 17373.7 & 1199.688 &       & $\mathbf{17408}$ & 7471.837 \\
          &       & 0.1   & 0.001 &       &\cellcolor[gray]{.9} $\mathbf{16266}$  &   \cellcolor[gray]{.9} 1.23s    &       &       & $\mathbf{16266}$ & 80    &       & $\mathbf{16266}$ & 44301629 &       & 16177.27 & 1229.167 &       & $\mathbf{16266}$ & 6806.991 \\
          &       &       & 0.01  &       &\cellcolor[gray]{.9}$\mathbf{17151}$   & \cellcolor[gray]{.9} 58.59s &       &       & $\mathbf{17151}$ & 672   &       & $\mathbf{17151}$ & 50929605 &       & 17090.43 & 1247.604 &       & $\mathbf{17151}$ & 6305.284 \\
          &       &       & 0.1   &       & \cellcolor[gray]{.9} $\mathbf{17366}$ & \cellcolor[gray]{.9}3.86s &       &       & $\mathbf{17366}$ & 688   &       & $\mathbf{17366}$ & 66415177 &       & 17338.23 & 1242.083 &       & $\mathbf{17366}$ & 8335.682 \\
          \hline
       uncorr 2   & 19545 & 0.01  & 0.001 &       & 29248 & 29248 & 29248 &       &$\mathbf{29250}$   & 610   &       &  $-$       &  $-$       &       & 29159.13 & 1276.042 &       & $\mathbf{29250}$ & 20153.41 \\
          &       &       & 0.01  &       &  \cellcolor[gray]{.9} $\mathbf{29508}$ &  \cellcolor[gray]{.9} 9.98s &       &       & $\mathbf{29508}$ & 658   &       & $-$        &  $-$       &       & 29401.57 & 1166.042 &       & $\mathbf{29508}$ & 19399.76 \\
          &       &       & 0.1   &       &\cellcolor[gray]{.9} $\mathbf{29595}$ & \cellcolor[gray]{.9} 2.87s &       &       & $\mathbf{29595}$  & 586   &       &   $-$      &   $-$      &       & 29512.8 & 1277.083 &       & $\mathbf{29595}$  & 20561.9 \\
          &       & 0.05  & 0.001 &       &\cellcolor[gray]{.9}$\mathbf{28766}$  &  \cellcolor[gray]{.9}2.45s     &       &       & $\mathbf{28766}$ & 691   &       &    $-$     &   $-$      &       & 28312.4 & 1276.563 &       & $\mathbf{28766}$ & 19775.83 \\
          &       &       & 0.01  &       &\cellcolor[gray]{.9} $\mathbf{29415}$  &   \cellcolor[gray]{.9} 5.78s    &       &       & $\mathbf{29415}$ & 584   &       &  $-$       &   $-$      &       & 29159.7 & 1278.125 &       & $\mathbf{29415}$ & 21976.75 \\
          &       &       & 0.1   &       & \cellcolor[gray]{.9} $\mathbf{29585}$  & \cellcolor[gray]{.9}4.6s  &       &       & $\mathbf{29585}$ & 1048  &       &  $-$       &   $-$      &       & 29446.37 & 1273.438 &       & $\mathbf{29585}$ & 20835.11 \\
          &       & 0.1   & 0.001 &       &\cellcolor[gray]{.9} $\mathbf{28472}$  &   \cellcolor[gray]{.9} 5.18s    &       &       & $\mathbf{28472}$ & 1050  &       &    $-$     &     $-$    &       & 28379.33 & 1273.958 &       & $\mathbf{28472}$ & 23583.94 \\
          &       &       & 0.01  &       & 29248 & 29248 &   29248    &       &  $\mathbf{29250}$  & 516   &       &    $-$     &   $-$      &       & 29117.5 & 1276.563 &       &  $\mathbf{29250}$  & 23508.19 \\
          &       &       & 0.1   &       & \cellcolor[gray]{.9}  $\mathbf{29508}$  &  \cellcolor[gray]{.9} 8.01s &       &       & $\mathbf{29508}$ & 581   &       &   $-$      &    $-$     &       & 29505.1 & 1276.042 &       & $\mathbf{29508}$ & 18264.97 \\
          \hline
        uncorr 3  & 36091 & 0.01  & 0.001 &       & 40858 &    105.48s   &       &       & 40858 & 1190  &       &    $-$     &   $-$      &       & 40785.7 & 1448.958 &       & 40858 & 59528.52 \\
          &       &       & 0.01  &       &\cellcolor[gray]{.9}  $\mathbf{41084}$  & \cellcolor[gray]{.9} 2.75s &       &       & $\mathbf{41084}$   & 881   &       &$-$         &     $-$    &       & 41022.8 & 1348.958 &       & $\mathbf{41084}$ & 70468.21 \\
          &       &       & 0.1   &       & \cellcolor[gray]{.9} $\mathbf{41121}$  & \cellcolor[gray]{.9} 4.3s  &       &       & $\mathbf{41121}$  & 1721  &       &   $-$      &  $-$       &       & 41069.2 & 1346.875 &       & 41120.53 & 64887.42 \\
          &       & 0.05  & 0.001 &       & $\mathbf{40482}$   & $\mathbf{40482}$ &  $\mathbf{40482}$     &       & $\mathbf{40482}$ & 2445  &       &   $-$      &     $-$    &       & 40124.27 & 1342.188 &       & $\mathbf{40482}$ & 67218.26 \\
          &       &       & 0.01  &       & \cellcolor[gray]{.9}$\mathbf{40985}$   & \cellcolor[gray]{.9} 58.57s &       &       & $\mathbf{40985}$  & 1562  &       &    $-$     &  $-$       &       & 40778.7 & 1346.354 &       & $\mathbf{40985}$  & 68263.79 \\
          &       &       & 0.1   &       & \cellcolor[gray]{.9}$\mathbf{41090}$  & \cellcolor[gray]{.9}4.91s &       &       & $\mathbf{41090}$ & 1121  &       &    $-$     &  $-$       &       & 40995.8 & 1348.958 &       & $\mathbf{41090}$ & 60929.71 \\
          &       & 0.1   & 0.001 &       & $\mathbf{40150}$ & $\mathbf{40150}$  & $\mathbf{40150}$  &       & $\mathbf{40150}$  & 1597  &       & $-$        &    $-$     &       & 40096.6 & 1341.667 &       & 40148.83 & 57693.78 \\
          &       &       & 0.01  &       &$\mathbf{40858}$  & $\mathbf{40858}$  & $\mathbf{40858}$ &       & $\mathbf{40858}$ & 1422  &       &  $-$       &   $-$      &       & 40845.53 & 1346.875 &       & $\mathbf{40858}$ & 58638.7 \\
          &       &       & 0.1   &       &\cellcolor[gray]{.9}  $\mathbf{41084}$ & \cellcolor[gray]{.9} 5.77s &       &       & $\mathbf{41084}$ & 1305  &       &  $-$       &  $-$       &       & 41059.5 & 1347.917 &       & $\mathbf{41084}$ & 58810.93 \\
          \bottomrule
    \end{tabular}}}%
  \label{tab:100Normal}%
\end{table}%

\begin{table}[t]
   \centering
  \caption{Statistic results for the instance eil101 with 500 items based on Normal distribution and Chebyshev's inequality}
  \scalebox{0.5}{
 \makebox[\linewidth][c]{
 \tabcolsep=0.1cm
    \begin{tabular}{crrrrrrrrrrlrrrl}
    \toprule
          & \multicolumn{1}{l}{Capacity} & \multicolumn{1}{l}{$\beta$} & \multicolumn{1}{l}{$\alpha$} &       & \multicolumn{3}{c}{ILP} &       & \multicolumn{3}{c}{(1+1)EA (1)} &       & \multicolumn{3}{c}{GSEMO (2)} \\
          &       &       &       &       & \multicolumn{1}{l}{2mins} & \multicolumn{1}{l}{6mins} & \multicolumn{1}{l}{10mins} &       & \multicolumn{1}{l}{profit} & \multicolumn{1}{l}{time(ms)} & stat  &       & \multicolumn{1}{l}{profits} & \multicolumn{1}{l}{time(ms)} & stat \\
          \midrule
           bou-s-c 1 & 61447 & 0.01  & 0.001 &       & $\mathbf{78649}$ & $\mathbf{78649}$ & $\mathbf{78649}$ &       & 75578.23 & 10723.1 & 2(-)  &       & 77967.43 & 64191.74 & 1(+) \\
          &       &       & 0.01  &       & 78649 & $\mathbf{78991}$ & $\mathbf{78991}$ &       & 77276.7 & 11453.57 & 2(-)  &       & 78614.93 & 65536.54 & 1(+) \\
          &       &       & 0.1   &       & $\mathbf{79471}$ & $\mathbf{79471}$ & $\mathbf{79471}$ &       & 77930.87 & 12345.55 & 2(-)  &       & 78775.83 & 65576.01 & 1(+) \\
          &       & 0.05  & 0.001 &       & $\mathbf{77515}$ & $\mathbf{77515}$ & $\mathbf{77515}$ &       & 66902.87 & 11525.26 & 2(-)  &       & 76868.6 & 64167.35 & 1(+) \\
          &       &       & 0.01  &       & 78649 & 78823 & $\mathbf{78897}$ &       & 74202.93 & 11487.28 & 2(-)  &       & 78222.2 & 64482.46 & 1(+) \\
          &       &       & 0.1   &       & 79095 & 79275 & $\mathbf{79280}$ &       & 76840.73 & 8534.408 & 2(-)  &       & 78664.87 & 65810.88 & 1(+) \\
          &       & 0.1   & 0.001 &       & 45784 & 73809 & $\mathbf{76045}$ &       & 58576.47 & 9455.142 & 2(-)  &       & 76069.83 & 64195.89 & 1(+) \\
          &       &       & 0.01  &       & $\mathbf{78649}$ & $\mathbf{78649}$ & $\mathbf{78649}$ &       & 70790.93 & 11302.55 & 2(-)  &       & 77959.23 & 64669.14 & 1(+) \\
          &       &       & 0.1   &       & 79024 & 79024 & $\mathbf{79108}$ &       & 75677.07 & 14396.65 & 2(-)  &       & 78590.93 & 65864.06 & 1(+) \\
          \hline
      bou-s-c 2    & 162943 & 0.01  & 0.001 &       & 190595 & 190630 & $\mathbf{190992}$ &       & 185699.2 & 14038.99 & 2(-)  &       & 189585.9 & 106708.7 & 1(+) \\
          &       &       & 0.01  &       & $\mathbf{192267}$ & $\mathbf{192267}$ & $\mathbf{192267}$ &       & 189245.8 & 13191.94 & 2(-)  &       & 190586.1 & 107000.5 & 1(+) \\
          &       &       & 0.1   &       & 192091 & 192445 & $\mathbf{192587}$ &       & 190476.9 & 12174.11 & 2(-)  &       & 190861.6 & 109597.3 & 1(+) \\
          &       & 0.05  & 0.001 &       & $\mathbf{189477}$ & $\mathbf{189477}$ & $\mathbf{189477}$ &       & 168475.9 & 12581.75 & 2(-)  &       & 187893.3 & 106447.4 & 1(+) \\
          &       &       & 0.01  &       & $\mathbf{191739}$ & $\mathbf{191739}$ & $\mathbf{191739}$ &       & 182869 & 12546.28 & 2(-)  &       & 190051.8 & 107138.7 & 1(+) \\
          &       &       & 0.1   &       & 192200 & 192340 & $\mathbf{192385}$ &       & 188430.7 & 12950.24 & 2(-)  &       & 190718.1 & 109228.8 & 1(+) \\
          &       & 0.1   & 0.001 &       & 187042 & $\mathbf{187826}$ & $\mathbf{187826}$ &       & 151777.8 & 12721.25 & 2(-)  &       & 186650.7 & 198824.9 & 1(+) \\
          &       &       & 0.01  &       & 190595 & 190595 & $\mathbf{190606}$ &       & 175823.3 & 12746.82 & 2(-)  &       & 189036.6 & 153722.1 & 1(+) \\
          &       &       & 0.1   &       & $\mathbf{192252}$ & $\mathbf{192252}$ & $\mathbf{192252}$ &       & 185978 & 12273.07 & 2(-)  &       & 190568.9 & 109722.5 & 1(+) \\
          \hline
       bou-s-c 3   & 307286 & 0.01  & 0.001 &       & 346237 & $\mathbf{346560}$ & $\mathbf{346560}$ &       & 338244.3 & 12581.56 & 2(-)  &       & 343847.3 & 258392 & 1(+) \\
          &       &       & 0.01  &       & 347893 & $\mathbf{348099}$ & $\mathbf{348099}$ &       & 343735.8 & 12276.5 & 2(-)  &       & 345168.6 & 177581.7 & 1(+) \\
          &       &       & 0.1   &       & 348242 & 348411 & $\mathbf{348417}$ &       & 345411 & 12445.5 & 2(-)  &       & 345577.9 & 176647.7 & 1(+) \\
          &       & 0.05  & 0.001 &       & 200259 & 343015 & $\mathbf{344569}$ &       & 310428.7 & 12381.02 & 2(-)  &       & 341663.5 & 161079.6 & 1(+) \\
          &       &       & 0.01  &       & 346222 & $\mathbf{346473}$ & $\mathbf{346473}$ &       & 333825.5 & 12646.84 & 2(-)  &       & 344489.5 & 156963.6 & 1(+) \\
          &       &       & 0.1   &       & 347997 & $\mathbf{348128}$ & $\mathbf{348212}$ &       & 342423.9 & 8960.956 & 2(-)  &       & 345352.1 & 152797.6 & 1(+) \\
          &       & 0.1   & 0.001 &       & 202187 & $\mathbf{341627}$ & $\mathbf{341627}$ &       & 283131.2 & 12028.49 & 2(-)  &       & 339939.8 & 301736 & 1(+) \\
          &       &       & 0.01  &       & 346271 & $\mathbf{346767}$ & $\mathbf{346767}$ &       & 322537.6 & 12357.88 & 2(-)  &       & 343863.1 & 156330.3 & 1(+) \\
          &       &       & 0.1   &       & $\mathbf{348020}$ & $\mathbf{348020}$ & $\mathbf{348020}$ &       & 338710.1 & 12106.73 & 2(-)  &       & 345162 & 156225.4 & 1(+) \\
          \hline
   uncorr 1 & 37686 & 0.01  & 0.001 &       & 87489 & $\mathbf{87515}$ & $\mathbf{87515}$ &       & 85923.03 & 10130.34 & 2(-)  &       & 87311.57 & 49110.62 & 1(+) \\
          &       &       & 0.01  &       & $\mathbf{88169}$ & $\mathbf{88169}$ & $\mathbf{88169}$ &       & 86916.9 & 11220.51 & 2(-)  &       & 87887.03 & 55903.9 & 1(+) \\
          &       &       & 0.1   &       & $\mathbf{88331}$ & $\mathbf{88331}$ & $\mathbf{88331}$ &       & 87354.8 & 11674.73 & 2(-)  &       & 88078.5 & 23058.93 & 1(+) \\
          &       & 0.05  & 0.001 &       & $\mathbf{86431}$ & $\mathbf{86431}$ & $\mathbf{86431}$ &       & 79797.07 & 11502.3 & 2(-)  &       & 86271.57 & 53259.75 & 1(+) \\
          &       &       & 0.01  &       & 87252 & 87729 & $\mathbf{87765}$ &       & 84955.23 & 11715.79 & 2(-)  &       & 87578.73 & 51830.28 & 1(+) \\
          &       &       & 0.1   &       & 88169 & $\mathbf{88236}$ & $\mathbf{88236}$ &       & 86697.47 & 12063.41 & 2(-)  &       & 87977.4 & 23219.8 & 1(+) \\
          &       & 0.1   & 0.001 &       & 84636 & $\mathbf{85738}$ & $\mathbf{85738}$ &       & 73279.37 & 13354.03 & 2(-)  &       & 85481.67 & 30391.65 & 1(+) \\
          &       &       & 0.01  &       & 16687 & $\mathbf{87513}$ & $\mathbf{87513}$ &       & 82380.23 & 12894.84 & 2(-)  &       & 87314.6 & 30685.76 & 1(+) \\
          &       &       & 0.1   &       & $\mathbf{88169}$ & $\mathbf{88169}$ & $\mathbf{88169}$ &       & 85934.5 & 11955.7 & 2(-)  &       & 87913.5 & 26798.67 & 1(+) \\
          \hline
       uncorr 2   & 93559 & 0.01  & 0.001 &       & $\mathbf{148834}$ & $\mathbf{148834}$ & $\mathbf{148834}$ &       & 146900.1 & 12483.14 & 2(-)  &       & 148232.8 & 42380.86 & 1(+) \\
          &       &       & 0.01  &       & 149445 & $\mathbf{149516}$ & $\mathbf{149516}$ &       & 148429.4 & 12386.6 & 2(-)  &       & 148769.8 & 42206.33 & 1(+) \\
          &       &       & 0.1   &       & 149726 & $\mathbf{149730}$ & $\mathbf{149730}$ &       & 148869 & 12523.58 & 2(-)  &       & 148951.5 & 38688.04 & 1(+) \\
          &       & 0.05  & 0.001 &       & 37741 & 147694 & $\mathbf{147780}$ &       & 139095.4 & 12240.31 & 2(-)  &       & 147102.9 & 41993.5 & 1(+) \\
          &       &       & 0.01  &       & 148834 & $\mathbf{148934}$ & $\mathbf{148934}$ &       & 145768.8 & 11206.09 & 2(-)  &       & 148419.6 & 42218.72 & 1(+) \\
          &       &       & 0.1   &       & 148834 & $\mathbf{149625}$ & $\mathbf{149625} $&       & 148125.6 & 14307.55 & 2(-)  &       & 148839.1 & 39514.93 & 1(+) \\
          &       & 0.1   & 0.001 &       & 101587 & 146281 & $\mathbf{146888} $&       & 130418.7 & 13508.59 & 2(-)  &       & 146250.9 & 41917.14 & 1(+) \\
          &       &       & 0.01  &       & 37741 & $\mathbf{148834} $& $\mathbf{148834} $&       & 142598.9 & 11929.05 & 2(-)  &       & 148163.2 & 41993.42 & 1(+) \\
          &       &       & 0.1   &       & 149398 & $\mathbf{149543} $& $\mathbf{149543} $&       & 147111.4 & 12399.31 & 2(-)  &       & 148791 & 39720.63 & 1(+) \\
          \hline
      uncorr 3    & 171819 & 0.01  & 0.001 &       & 208811 & $\mathbf{209074} $& $\mathbf{209074} $&       & 207151.1 & 12162.4 & 2(-)  &       & 206872.6 & 49345.98 & 1(+) \\
          &       &       & 0.01  &       & 209630 & $\mathbf{209630} $ & $\mathbf{209630} $ &       & 208489.4 & 13234.64 & 2(-)  &       & 207338.3 & 49400.53 & 1(+) \\
          &       &       & 0.1   &       & 209790 & $\mathbf{209790}$ & $\mathbf{209790} $&       & 208946.9 & 14540.8 & 2(-)  &       & 207487.1 & 49227.39 & 1(+) \\
          &       & 0.05  & 0.001 &       & 141224 & 208033 & $\mathbf{208095} $&       & 199039.4 & 12712.34 & 2(-)  &       & 205892.9 & 49377.69 & 1(+) \\
          &       &       & 0.01  &       & 209229 & $\mathbf{209289} $& $\mathbf{209289}$ &       & 205913.6 & 12785.77 & 2(-)  &       & 206973 & 49178.8 & 1(+) \\
          &       &       & 0.1   &       & 209663 & $\mathbf{209694}$ & $\mathbf{209694} $ &       & 208144.3 & 12656.35 & 2(-)  &       & 207608.4 & 49367.14 & 1(+) \\
          &       & 0.1   & 0.001 &       & 141356 & 144160 & $\mathbf{207294}$ &       & 188980.9 & 12458.81 & 2(-)  &       & 205302.1 & 48909.96 & 1(+) \\
          &       &       & 0.01  &       & 141561 & $\mathbf{209047}$ & $\mathbf{209047}$ &       & 202708.7 & 13160.75 & 2(-)  &       & 206906.2 & 49180.67 & 1(+) \\
          &       &       & 0.1   &       & 209630 & $\mathbf{209630}$ & $\mathbf{209630}$ &       & 207203.3 & 12552.59 & 2(-)  &       & 207316.4 & 49094.54 & 1(+) \\
          \bottomrule
    \end{tabular}}}%
  \label{tab:500Normal}%
\end{table}%

In the last category of instance, we consider instances for which the weights of items conform to Normal distributions with known expectations and variances of weights. In all instances, the variances of weights are expressed as a percentage of the expected weights: $\sigma_i^2 =a_i * \beta$. In this arrangement, according to the theorem of the Chernoff bound, the upper bound of the chance constraint cannot be calculated using the Chernoff bound, and we only apply Chebyshev's inequality. The results are listed in Table \ref{tab:100Normal} and Table \ref{tab:500Normal}, and we mark the best results across all approaches in each row in bold.

The \textit{profit} column lists the mean value for 30 runs, and the  \textit{stat} column provides statistical test results baes on the comparison of the performances of (1+1)~EA and GSEMO. It can be observed from the tables that algorithms produce inferior solutions when the value of $\beta$ increases or the value of $\alpha$ increases. The results show that the heuristic approach outperforms the other algorithms for instances with 100 items, but takes more than two hours for instances with 500 items. Moreover, under the condition in which algorithms are run at similar computation times, the results solved by GSEMO are better than those solved using ILP. Furthermore, the performance of GSEMO is significantly better than that of (1+1)~EA.

By investigating the results lists in all tables, we found when limited the computation time into decades seconds, the GSEMO can provide high quality solutions overall. Moreover, in an acceptable running time (minutes),Heuristic approach performs better than other approaches in small size instances, but can't obtain a solution within two hours for large size instances. However, in this paper, (1+1)~EA and GSEMO are the simple evolutionary algorithms, in future studies, it might be possible to investigate the performance of different evolutionary algorithms on solving the chance-constrained knapsack problem.

\section{Conclusion}
\label{sec:conclusion}
Chance-constrained optimization problems play a crucial role in various real-world applications as they allow to limit the probability of violating a given constraint when dealing with stochastic problems. We have considered the chance-constrained knapsack problem and shown how to incorporate well-known probability tail inequalities into the search process of an evolutionary algorithm. Furthermore, we introduced deterministic approaches and compared them to the designed evolutionary algorithms. For the deterministic approaches, the disadvantage of them is the computation time when dealing with the chance-constrained knapsack problem. Our experimental results show that when solving large-size instances, DP can not finish in 10 hours, The heuristic approach can not finish in two hours while EAs can obtain a good quality solution in minutes. Even for small-size instances, GSEMO can get similar quality solutions as ILP and the heuristic approach but takes less computation time in most instances. Furthermore, our investigations point out under what circumstances Chebyshev's inequality or Chernoff bounds are favored as part of the fitness evaluation. The Chernoff bound is preferable in cases where the probability of violating the capacity of knapsack is small. We also have shown that using a multi-objective approach when dealing with the chance-constrained knapsack problem provides a clear benefit compared to its single-objective formulation for all kinds of investigated instance classes.

\section{Acknowledgements}
This work has been supported by the Australian Research Council through grant DP160102401 and by the South Australian Government through the Research Consortium "Unlocking Complex Resources through Lean Processing".

\bibliographystyle{abbrv}
\bibliography{main}
\end{document}